\RequirePackage{fix-cm}
\documentclass[twocolumn]{svjour3}          % twocolumn
\usepackage[toc,page]{appendix}
\smartqed  % flush right qed marks, e.g. at end of proof

\usepackage{times}
\usepackage{helvet}
\usepackage{courier}
\usepackage{comment}
\usepackage{subcaption}
\usepackage{wrapfig}
\usepackage{float}
\usepackage{times}

\usepackage{epsfig}
\usepackage{graphicx}
\usepackage{amsmath}
\usepackage{wrapfig}
\usepackage{amssymb}
\usepackage{comment}
\usepackage[font=small,labelfont=bf]{caption}
\usepackage{enumi tem}
\newcommand{\ignore}[1]{}
\usepackage[noend]{algorithm2e}
\usepackage{algorithmicx,algpseudocode}
\usepackage{booktabs,tabularx}%
\usepackage{multirow} 
\usepackage{amsthm}
\theoremstyle{plain}
\theoremstyle{definition}
\newtheorem{thm}{Theorem}[section]
\newtheorem{lem}[thm]{Lemma}

\theoremstyle{definition}

\theoremstyle{remark}

\def\ignore#1{}

\def\onedot{.}
\def\eg{\emph{e.g}\onedot~} 
\def\ie{\emph{i.e}\onedot~}

\def\etal{\emph{et al}~}

\begin{document}

%\title{Overlapping Domain Cover for  Scalable  and Accurate Regression Kernel Machines }

\title{Overlapping Cover Local Regression Machines }


%\subtitle{from Pure Text or weak Attributes}

\author{Mohamed Elhoseiny    \and
        Ahmed Elgammal     %etc.
}

\authorrunning{M Elhoseiny et al.}

\institute{Mohamed Elhoseiny$^{1}$ and Ahmed Elgammal$^{2}$ \at
	       $^{1}$Facebook AI Research\\
	       $^{2}$Department of Computer Science, Rutgers University\\
             % 110 Frelinghuysen Road, Piscataway, NJ 08854-8019, USA \\
              %USA\\
              %Tel.: +1-732-208-9712\\	
              \email{elhoseiny@fb.com, elgammal@cs.rutgers.edu}           %  \\
       %    \and
        %   Ahmed Elgammal \at
         %  110 Frelinghuysen Road, \\ 
         %  Piscataway, NJ 08854-8019 \\
         %  USA\\
         %  \email{elgammal@cs.rutgers.edu}          
}

\date{Received: date / Accepted: date}

\maketitle

\begin{abstract}

We present the Overlapping Domain Cover (ODC) notion for kernel machines, as a set of overlapping subsets of the data that covers the entire training set and optimized to be spatially cohesive as possible. We show how this notion benefit the speed of local kernel machines for regression in terms of both speed while achieving while minimizing the prediction error. We propose an efficient ODC framework, which is applicable to various regression models and in particular reduces the complexity of Twin Gaussian Processes (TGP) regression from cubic to quadratic. Our notion is also applicable to several kernel methods (\eg   Gaussian Process Regression(GPR) and IWTGP regression, as shown in our experiments). We also theoretically justified the idea behind our method to improve local prediction by the overlapping cover. We validated and analyzed our method on three benchmark human pose estimation datasets and interesting findings are discussed.  

%Section 4, presents an idea of how to build a domain decomposition TGP 

\end{abstract}
\section{Introduction}
\label{sec:1}

\ignore{One of the main problems, in many computer vision application, is to estimate a continuous real-valued function or a structured-output function from input features.} Estimation of a continuous real-valued or a structured-output function from input features is one of the critical problems that appears  in many machine learning  applications. Examples include predicting the joint angles of the human body from images, head pose, object viewpoint, illumination direction, and a person's age and gender. Typically, these problems are formulated by a regression model.  Recent advances in structure regression encouraged researchers to adopt it for formulating various problems with high-dimensional output spaces, such as segmentation, detection, and image reconstruction, as regression problems. However, \ignore{they are constrained by }the computational complexity of the state-of-the-art regression algorithms limits their applicability for big data\ignore{ that are used to compute the predictions}. In particular,  kernel-based regression algorithms such as Ridge Regression~\cite{ridgeReg70}, Gaussian Process Regression (GPR)~\cite{Rasmussen:2005}, and the Twin Gaussian Processes (TGP)~\cite{Bo:2010} require inversion of kernel matrices ($O(N^3)$, where $N$ is the number of the training points), which limits their applicability for big data. We refer to these non-scalable versions of GPR and TGP as full-GPR and full-TGP, respectively.

Khandekar et. al. \cite{khandekar2014advantage} discussed properties and  benefits of overlapping clusters for minimizing the conductance from spectral perspective. These properties of overlapping clusters also motivate studying scalable local prediction based on overlapping kernel machines.  Figure~\ref{fig:overlapandnonoverlap_ex} illustrates the notion by starting from a set of points, diving them into either disjoint and overlapping subsets, and finally learning a kernel prediction function on each (i.e., $f_i(x^*)$ for subset $i$, $x^*$ is testing point). \ignore{Our method is accurate and scalable as demonstrated in the experiments on 3-datasets (including Human3.6M, the largest human-3D-pose dataset).}
\textit{In summary, the main question, we  address in this paper, is how local kernel machines with overlapping training data could help speedup the computations and gain accurate predictions. We achieved considerable speedup and good performance on GPR, TGP, and IWTGP (Importance Weighted TGP) applied to 3D pose estimation datasets\ignore{; our framework firstly achieves quadratic prediction complexity for TGPs}. To the best of our knowledge, our framework is the first to achieve quadratic prediction complexity for TGP. The ODC concept is also novel in the context of kernel machines and is shown here to be successfully applicable to multiple kernel-machines\ignore{, resulting in a considerable speedup and an accurate prediction}}. We studies in this work GPR and TGP and IWTGP  (a third model) kernel machines\ignore{ (denoted by SM in the rest of the paper)}. The remainder of this paper is organized as follows: Section ~\ref{sec:2}  and~\ref{sec:relappmethod} presents some motivating kernel machines and the related work. Section ~\ref{sec:3} presents our approach and a theoretical justification for our ODC concept. Section ~\ref{sec:6} and \ref{sec:7} presents our experimental validation and conclusion.

\begin{figure*}[t]
\centering
%\centering
%\begin{tabular}{cc}
%\bmvaHangBox{\fbox{
\includegraphics[width=0.31\textwidth]{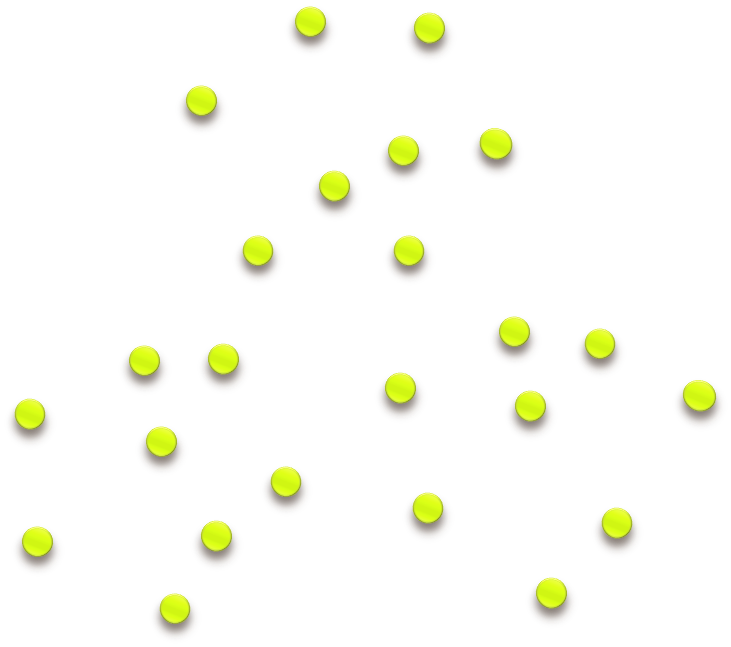}
%\bmvaHangBox{\fbox{
\includegraphics[width=0.31\textwidth]{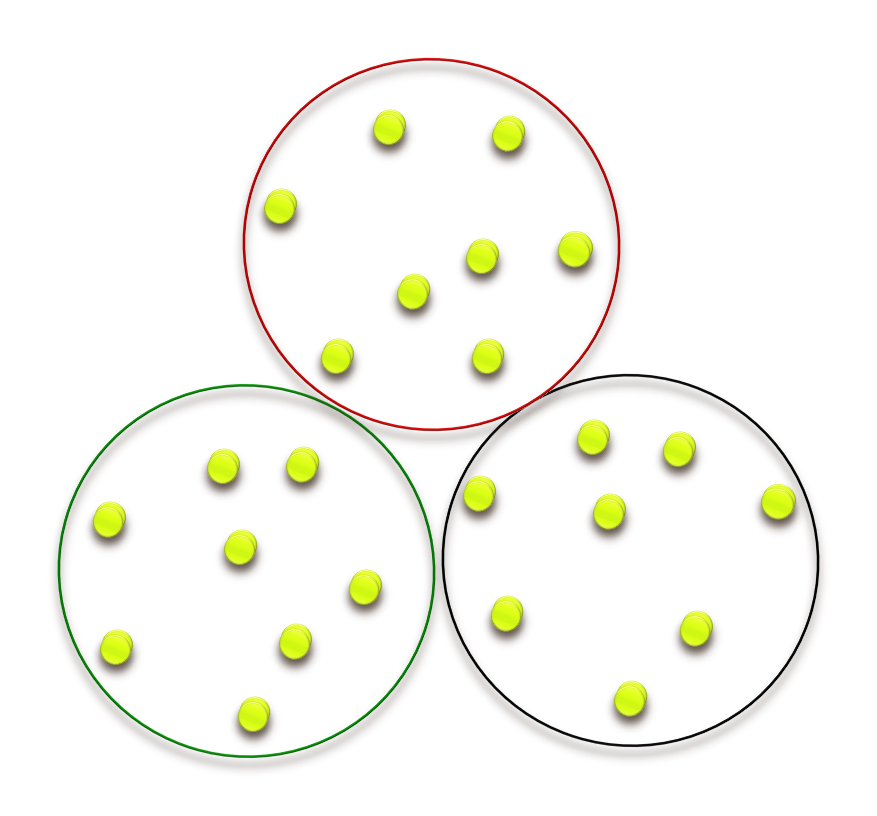}%\\
\includegraphics[width=0.31\textwidth]{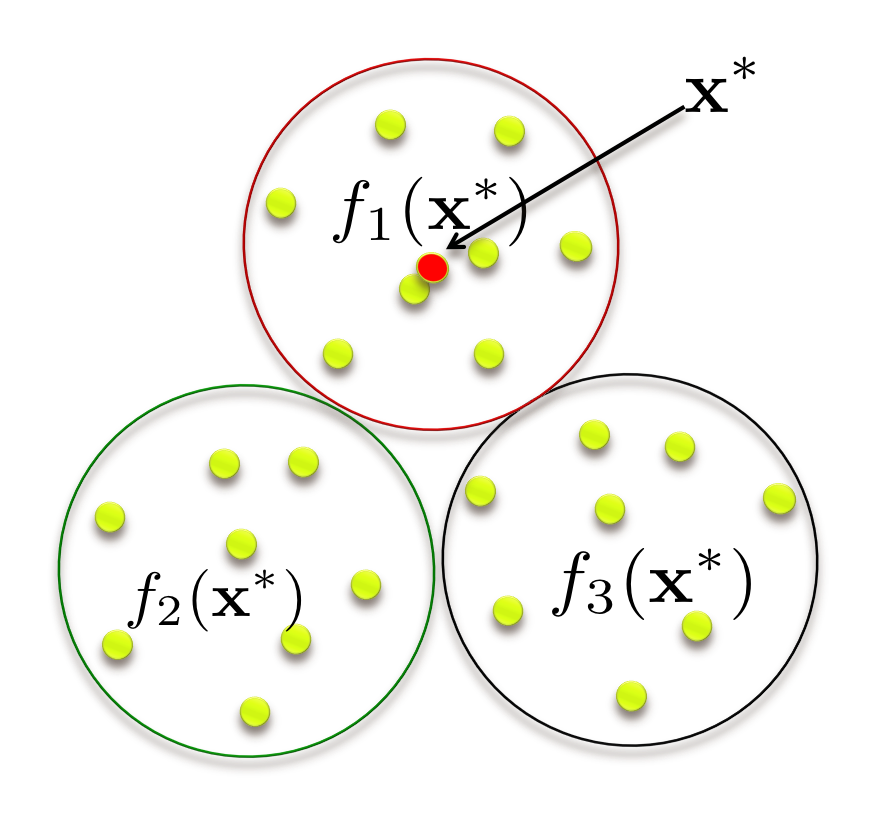}
\includegraphics[width=0.62\textwidth]{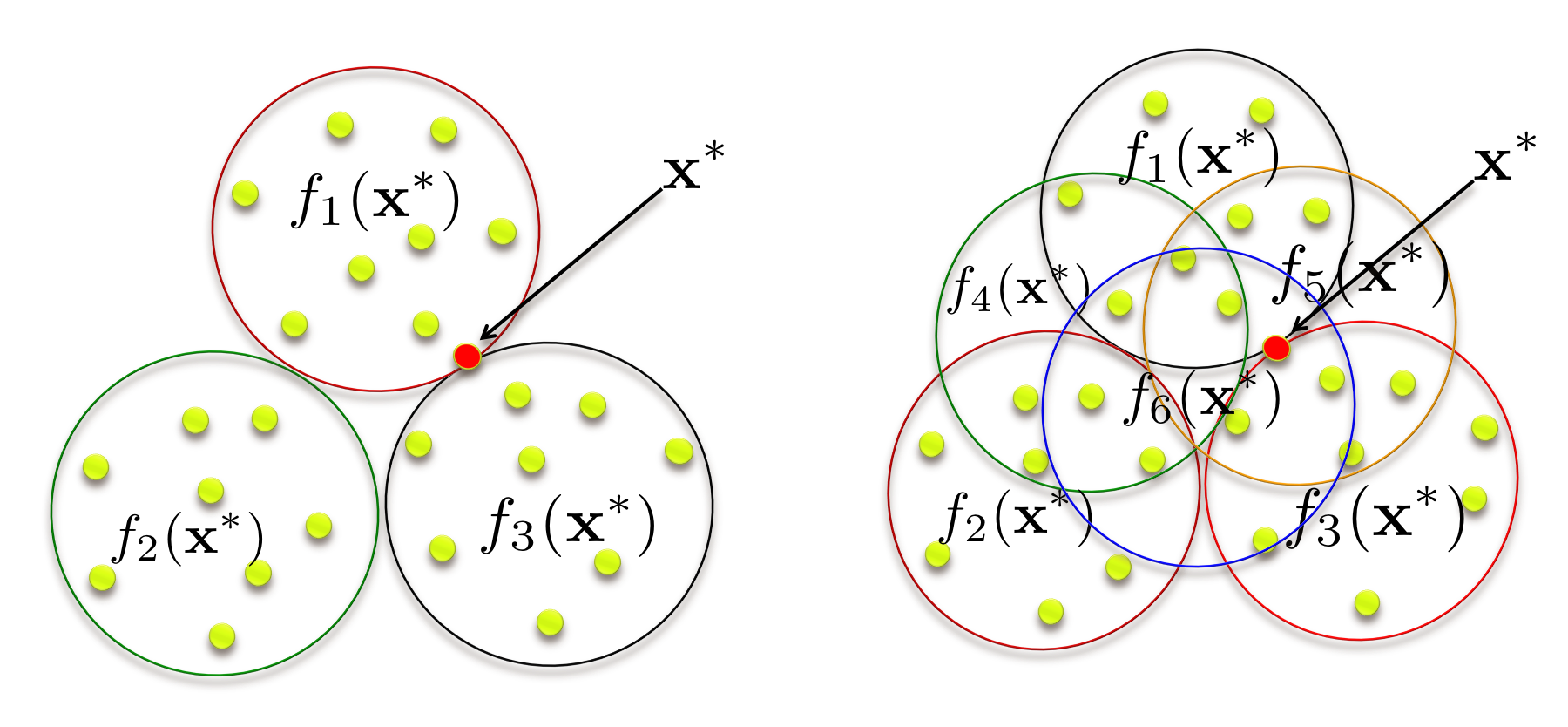}
%(a)&(b)s
%\end{tabular}
\caption{\textbf{Top}: Left:24 points, Middle: Overlapping Cover, Right: disjoint kernel machines of 8  points (evaluating $x^*$ near a middle of a kernel machine).  \textbf{Bottom}: Left: disjoint kernel machine evaluation on boundary),  Right: 6 Overlapping kernel machines of 8 points.  $f_i(\mathbf{x}^*)$ is the $i^{th}$ kernel machine prediction for $\mathbf{x}^*$ test point.}
\label{fig:overlapandnonoverlap_ex}
\vspace{-3mm}
\end{figure*}

\section{Background on Full GPR and TGP Models}
\label{sec:2}
In this section, we show example kernel machines that motivated us to propose the ODC framework to improve their performance and scalability.   Specifically, we review GPR for single output regression, and TGP for structured output regression.\ignore{ Specifically, we review two powerful machines for single output (Gaussian Process Regression(GPR)) and structured output regression (Twin Gaussian Processes(TGP)).Then, we present an importance weighting algorithm to resolve data bias under Twin-Gaussian Processes (IWTGP)} We selected GPR and TGP kernel machines for their increasing interest and impact. However, our framework is not  restricted to them\ignore{ as it depends on the overlapping domain cover notion that could be generally adopted}. \ignore{We conclude this section by showing practical limitation of these machines that inspires our method. }

%\subsection{Gaussian Process Regression (GPR)}
%\label{sec:relgpr}

\textbf{GPR}~\cite{Rasmussen:2005}  assumes a linear model in the kernel space with Gaussian noise in a single-valued output, i.e., $y = f(\textbf{x})  + \mathcal{N} (0, \sigma_{n}^2)$, where $\textbf{x} \in \mathbb{R}^{d_X}$ and ${y} \in \mathbb{R}$. Given a training set $\lbrace \textbf{x}_i, {y}_i , i =1:N \rbrace$, the posterior distribution of $y$ given a test point $\textbf{x}_{*}$ is:   \ignore{dimension $i$ of the classifier to be predicted ($\textbf{y}_* \in R^{d_Y}$)}
\begin{comment}
\begin{equation}
y_i = f_i(t) +e_i ,  \,\,\,\,     e_i ∼ \mathcal{N} (0, \sigma_{ni}^2),        
\end{equation}
\end{comment}
\begin{equation}
\small
\begin{split}
p(y|\textbf{x}_*)=   \mathcal{N} (&\boldsymbol{\mu}_{y} =  \textbf{k}{(\textbf{x}_*)}^\top (\textbf{K} + \sigma^2_{n} \textbf{I})^{-1} \mathbf{f},   \\  &  \sigma_{y}^2 = k({\textbf{x}_*, \textbf{x}_*}) -\textbf{k}{(\textbf{x}_*)}^\top (\textbf{K} + \sigma^2_{n} \textbf{I})^{-1} \textbf{k}{(\textbf{x}_*)})
\end{split}
\label{eq:gprpred}
\end{equation}
where $k(\textbf{x},\textbf{x}')$ is kernel defined in the input space, $\mathbf{K}$ is an $N \times N$ matrix, such that $\mathbf{K} (l,m) = k(\textbf{x}_l, \textbf{x}_m)$, $\textbf{k}{(\textbf{x}_*)}  = [k(\textbf{x}_*,\textbf{x}_1),, ..., {k}(\textbf{x}_*,\textbf{x}_{N})]^\top$, $\textbf{I}$ is an  identity matrix of size $N$, $\sigma_{n}$ is the variance of the measurement noise, $\mathbf{f} = [y_1,\cdots  $ $, y_N]^\top$. GPR could predict structured output $\textbf{y} \in \mathbb{R}^{d_Y}$ by training a GPR  model for each dimension\ignore{ $i = 1 : d_Y$}. However, this indicates that GPR does not capture dependency between output dimensions which limit its performance. 
%\subsection{Twin Gaussian Processes (TGP)}
%\label{sec:reltgp}

\textbf{TGP}~\cite{Bo:2010} encodes the relation between both inputs and outputs using GP priors. This was achieved by minimizing the Kullback-Leibler divergence between the marginal GP of outputs (e.g., poses) and observations (e.g., features)\ignore{; we refer the reader to~\cite{Bo:2010} for derivation}. Hence,  TGP prediction is given by:\ignore{ the estimated pose in TGP is given as the solution of the following optimization problem:}
\begin{equation}
\small
\begin{split}
\hat{\textbf{y}}(\textbf{x}_*) =  &\underset{\textbf{y}}{\operatorname{argmin       }}[  k_Y(\textbf{y},\textbf{y})  -2 \textbf{k}_Y(\textbf{y})^\top (\textbf{K}_X + \lambda_X  \textbf{I})^{-1} \textbf{k}_X(\textbf{x}_*) \\&- \eta  log (k_Y(\textbf{y},\textbf{y})  -\textbf{k}_Y(y)^\top ({\textbf{K}_Y}+ \lambda_Y \textbf{I})^{-1} {\textbf{k}_Y(\textbf{y})} ) ]
\end{split}
\label{eq:tgp}
\vspace{-5mm}
\end{equation}
where $\eta  = k_X(\textbf{x}_*,\textbf{x}_*) -\textbf{k}_X(\textbf{x}_*)^\top  (\textbf{K}_X + \lambda_X  \textbf{I})^{-1} \textbf{k}_X($ $\textbf{x}_*)$, $k_X(\textbf{x},\textbf{x}') = exp(\frac{- \|\textbf{x}-\textbf{x}' \|}{2 \rho_x^2})$ and $k_Y(\textbf{y},\textbf{y}')$ $= exp(\frac{- \|\textbf{y}-\textbf{y}' \|}{2 \rho_y^2})$ are Gaussian kernel functions for input feature $\textbf{x}$ and output vector $\textbf{y}$, $\rho_x$ and $\rho_y$ are the kernel bandwidths for the input and the output  . $\textbf{k}_Y(\textbf{y}) = [k_Y(\textbf{y},\textbf{y}_1), ..., k_Y(\textbf{y},\textbf{y}_{N})]^\top$, where $N$ is the number of the training examples. $\textbf{k}_X(\textbf{x}_*) = [k_X(\textbf{x}_*,\textbf{x}_1), ...,$ ${k}_X(\textbf{x}_*,\textbf{x}_{N})]^\top$, and $\lambda_X$ and $\lambda_Y$ are regularization parameters to avoid overfitting. This optimization problem can be solved using a quasi-Newton optimizer with cubic polynomial line search \ignore{for optimal step size selection}~\cite{Bo:2010}; we denote the number of steps to convergence as $l_2$.

\begin{table*}[t!]
\centering
 %\vspace{-5mm}
\caption{Comparison of computational Complexity of training and testing for each of Full, NN (Nearest Neighbor), FITC, Local-RPC, and our ODC. Training is the time include all computations that does not depend on test data, which includes clustering in some of these methods. Testing includes computations only needed for prediction\ignore{performed at test time, i.e., Mean $\textbf{Y}$ and variance for GPR kernel machine and the prediction $\textbf{Y}$ for TGP kernel machine}}
%\vspace{2mm}
 \label{tab:thcomp}
  \scalebox{0.68}{
    \begin{tabular}{|l|ccc|cc|c|}
    %\toprule
    \hline
          &  \multicolumn{3}{c|}{\textbf{Training for GPR and TGP}}    & \multicolumn{3}{|c|}{\textbf{Testing for each point} }  \\ \hline
    %\midrule
          & \textbf{Ekmeans Clustering} & \textbf{RPC Clustering}   &    \textbf{Model training}   & \textbf{GPR-Y} & \textbf{GPR-Var} & \textbf{TGP-Y} \\ \hline 
    \textbf{Full} & -     & -     & $O(N^3 + N^2 d_X)$    & $O(N \cdot (d_X +d_Y)$ & $O(N^2 \cdot d_Y)$ & $O(l_2 \cdot N^2 \cdot d_Y)$ \\
    \textbf{NN {\cite{Bo:2010}}} & -     & -     & -     & $O(M^3 \cdot d_Y)$ & $O(M^3 \cdot d_Y)$ & $O(M^3 + l_2 \cdot M^2 \cdot d_Y)$ \\
    \textbf{FIC (GPR only, $d_Y=1$ {~\cite{fic06}})} & - & - & $O(M^2 \cdot ( {N} + d_X))$  & $O(M \cdot d_X)$  & $O(M^2)$ & -  \\
    \textbf{Local-RPC (only GPR, $d_Y=1${ \cite{Chalupka:2013}})} & - & $N \cdot log(\frac{N}{M})$ & $O(M^2 \cdot ( {N} + d_X) )$  & $O(M \cdot d_X)$ & $O(M^2)$ & -  \\
     %\textbf{SoD} & - & - & O($M^3 \cdot d_X$)  & O($M \cdot d_X$) & O($M^2 $) & -  \\
     \textbf{ODC (our framework)} & $O(N \cdot \frac{N}{(1-p) M} \cdot d_X \cdot l_1)$ & $O(N \cdot log(\frac{N}{(1-p) M}) \cdot d_X)$ & O($M^2 \cdot(\frac{N}{1-p} + d_X)$)  & $O(K' \cdot M \cdot( d_X + d_Y))$ & $O(K' \cdot M^2 \cdot d_Y)$ & $O(l_2 \cdot K' \cdot M^2 \cdot d_Y)$ \\ \hline 
    %\bottomrule
    \end{tabular}}%
      \vspace{-2mm}
\end{table*}%

\section{Importance Weighted Twin Gaussian Processes (IWTGP)}
\label{ss:cstgp}

Yamada et al~\cite{Yamada:2012} proposed the importance-weighted variant of twin Gaussian processes \cite{Bo:2010} called IWTGP.  The weights are calculated using RuLSIF \cite{YamadaSKHS11} (relative unconstrained least-squares importance fitting). The weights were modeled as $w_{\alpha}(\textbf{x},\boldsymbol{\theta}) = \sum_{l=1}^{n_{te}} \theta_l k(\textbf{x}, \textbf{x}_l)$ to minimize $E_{p_{te}(x)} [\,(w_{\alpha}(x,\mathbf{\theta})-w_{\alpha}(\textbf{x}))^2\,]$. where $k(\textbf{x},\textbf{x}_l) = exp(-\frac{\|\textbf{x}-\textbf{x}_l\|}{2 \tau^2})\,\,\,$, $w_{\alpha}(\textbf{x}) =\,\,\,$ $\frac{p_{te}(\textbf{x})}{(1-\alpha) p_{te}(\textbf{x}) +\alpha p_{tr}(\textbf{x})}$, $0\leq\alpha \leq 1$. To cope with this instability issue, setting $\alpha$ to $0 \le\alpha \le 1$ is practically useful for stabilizing the covariate shift adaptation, even though it cannot give an unbiased model under covariate shift \cite{YamadaSKHS11}. According  \cite{Yamada:2012} the optimal $\boldsymbol{\hat{\theta}}$ vector is computed in a closed form solution as follows.to
\begin{equation}
\boldsymbol{\hat{\theta}}= ({\hat{\textbf{H}}} + \nu \textbf{I} )^{-1} {\hat{\textbf{h}}}
\end{equation} 
where  $\hat{\textbf{H}}_{l,l'} = \frac{1-\alpha}{n_{te}}  \sum_{i=1}^{n_{te}} k(\textbf{x}_i^{te}, \textbf{x}_l^{te} k(\textbf{x}_i^{te},\textbf{x}_{l'}^{te}) + $ $\frac{\alpha}{n_{tr}} \sum_{j=1}^{n_{tr}}  $ $ k(\textbf{x}_j^{tr}, \textbf{x}_l^{te} k(\textbf{x}_j^{tr},\textbf{x}_{l'}^{te})$, $\hat{\textbf{h}}$ is an $n_{te}$- dimensional vector with the $l^{th}$ element $\hat{\textbf{h}}_l = \frac{1}{n_{te}} \sum_{i=1}^{n_{te}} k(\textbf{x}_i^{te}, \textbf{x}_l^{te})$, $\textbf{I}$ is an $n_{te}\times n_{te}$-dimensional identity matrix. where $n_{te}$ and $n_{tr}$ and the number of testing and training points respectively. Model selection of RuLSIF is based on cross-validation with respect to the squared-error criterion $J$ in \cite{YamadaSKHS11}. Having computed $\boldsymbol{\hat{\theta}}$, each input and output examples are simply re-weighted by $w_{\alpha}^{\frac{1}{2}}$ \cite{Yamada:2012}. Therefore, the output of the importance
weighted TGP (IWTGP) is given by
\begin{equation}
\begin{split}
\hat{y} =  \underset{y}{\operatorname{argmin       }}[ & K_Y(\textbf{y},\textbf{y}) -2 k_y(\textbf{y})^T \textbf{u}_w - \eta_w  log (K_Y(\textbf{y},\textbf{y}) -\\& k_y(\textbf{y})^T \textbf{W}^\frac{1}{2} (\textbf{W}^\frac{1}{2} \textbf{K}_Y \textbf{W}^\frac{1}{2} +  \lambda_y I)^{-1} \textbf{W}^\frac{1}{2} k_y(\textbf{y}) ) ]
\end{split}
\label{eq:IWTGP}
\end{equation}
where $\textbf{u}_w = \textbf{W}^\frac{1}{2}  (\textbf{W}^\frac{1}{2} \textbf{K}_X \textbf{W}^\frac{1}{2} + \lambda_x I)^{-1} \textbf{W}^\frac{1}{2} k_x(\textbf{x})$, $\eta_w = k_X(\textbf{x},\textbf{x}) - k_x(\textbf{x})^T \textbf{u}_w$. Similar to TGP, IWTGP can also be solved using a second order, BFGS quasi-Newton optimizer with cubic polynomial line search for optimal step size selection.

 Table~\ref{tab:thcomp} shows the training an testing complexity of full GPR and TGP models, where $d_Y$ is the dimensionality of the output. Table~\ref{tab:thcomp} also summarizes the computational complexity of the related approximation methods, discussed in the following section, and our method. N\ignore{The hyper-parameters of both GPR and TGP were trained by cross validation, similar to ~\cite{Bo:2010}; the SM include the learnt hyper-parameters.}\ignore{Yamada et al \cite{Yamada:2012} proposed the importance-weighted variant of twin Gaussian processes \cite{Bo:2010} called IWTGP.  The weights are calculated using RuLSIF \cite{YamadaSKHS11}. There are more details about using this approach in our framework in the attached SM. }.  %It is important to note that, there exist several existing kernel methods such as KTA~\cite{}, HSIC~\cite{}, W-KNN~\cite{}.  We focus on studing our ODC notion on GPR and TGPs (TGP and IWTGP) since they are among the most popular and best perfoming kernel machines. However,  we think the notion is applicable to other methods.  

\section{Related Work on Approximation  Methods}
\label{sec:relappmethod}

\begin{comment}
\hspace{0.08cm}
\begin{minipage}[b]{0.35\linewidth}
\centering
  \includegraphics[width=0.7\textwidth,height=0.5\textwidth]{figMInvSpeedUp.eps}
  \vspace{1mm}
  \caption{Speedup of ODC framework prediction on either TGP or GPR while retrieving precomputed matrix inverses as $M$ increases, compared with computing them on test time by KNN scheme (log-log scale)}
  \label{fig:SpeedUp}
\end{minipage}
%\vspace{-10mm}
\end{figure*}
\begin{figure}
\centering
 \includegraphics[width=0.25\textwidth,height=0.2\textwidth]{Ekmeans57.png}
 \vspace{4mm}
\caption{Assign and Balance-EKmeans on 300,000 random 2D points, K= 57 (best seen electronically).}
\end{figure}
\end{comment}

Various approximation approaches have  been presented to reduce the computational complexity in the context of GPR. As detailed in \cite{park11}, approximation methods on Gaussian Processes may be categorized into three trends: matrix approximation,
likelihood approximation, and localized regression. The matrix approximation trend is inspired by the observation that the kernel matrix inversion is the major part of the expensive computation, and thus, approximating the matrix by a lower rank version, $M \ll N$\ignore{,   helps reduce the computational demand} (e.g.,  Nystr\"{o}m  Method~\cite{Nystrom01}). While this approach reduces the computational complexity from $O(N^3)$ to $O(N M^2)$ for training, there is no guarantee on the non-negativity of the predictive variance~\cite{Rasmussen:2005}. In the second trend,  likelihood approximation is performed on testing and training examples, given $M$ artificial examples known as inducing inputs, selected from the training set (e.g.,  Deterministic Training Conditional (DTC)~\cite{DTC03}, Full Independent conditional (FIC)~\cite{fic06}, Partial Independent Conditional (PIC)~\cite{pic07}). The drawback of this trend is \ignore{that it deals with }the dilemma of selecting $M$ inducing points, which might be distant from the test point, resulting in a performance decay; see Table~\ref{tab:thcomp} for the complexity of FIC.

\ignore{However, } A third trend, localized regression, is based on the belief that distant observations are almost unrelated\ignore{, and adopted in our work}. The prediction of a test point is achieved through its $M$ nearest points\ignore{ from the training set}. One technique to implement this notion is through decomposing the training points into disjoint clusters during training, where prediction functions are learned for  each of them~\cite{park11}. At test time, the prediction function of the closest
cluster\ignore{, where the test point belongs, } is used to predict the corresponding output. While this method is efficient\ignore{ and adaptive to non-stationary change}, it introduces discontinuity problems on boundaries of the subdomains. Another way to implement local regression is through Mixture of Experts (MoE) as an Ensemble method to make prediction based on computing the final output by combining outputs of local predictors called experts (see a study on MoE methods~\cite{tymoe12}). Examples include Bayesian committee
machine (BCM~\cite{BCM00}), local probabilistic regression (LPR~\cite{LPR08}), mixture of Tree of Gaussian Processes (GPs)~\cite{TreeGPs07}, and
Mixture of GPs~\cite{Rasmussen:2005}. While these approaches overcome the discontinuity problem by the combination mechanism, they suffer from intensive complexity at test time, which limits its applicability in large-scale setting\ignore{. One more computational aspect is that mixture models, as }, e.g., Tree of GPs and Mixture of GPs, involve complicated integration, approximated by computationally expensive sampling or Monte Carlo simulation. 

\ignore{
Park  et al~\citet{park11} proposed an approach for fast computation of GPR with a focus on large spatial data sets. The approach decomposes the domain of a single output regression function into small subdomains, inspired by a domain decomposition  approach for solving  Partial differential Equations (PDEs). It then infers a local regression function for each subdomain. In contrast to prior methods, this approach is easier to parallelize. However, it was mainly based on building consistent boundary value functions between subdomains on  a regular grid. Hence, their approach was evaluated on data of maximum input dimension $2$.}

Park etal.~\cite{park11} proposed a large-scale approach for  GPR by domain decomposition on up to 2D grid on input, where a local regression function is inferred for each subdomain such that they are consistent on boundaries.\ignore{The approach model the problem by solving Partial Differential Equations to maximize consistent prediction on the boundaries of sudomains, defined on a regular grid (upto 2D grid).} This approach obviously lacks a solution to high-dimensional input data because the size of the grid increases exponentially with the dimensions, which limits its applicability\ignore{makes the approach unapplicable in is not applicable to many applications}. More recently,\ignore{  Chalupka et al}~\cite{Chalupka:2013} proposed a Recursive Partitioning Scheme (RPC) to decompose the data into non-overlapping equal-size clusters, and they built a GPR on each cluster\ignore{; we denote this method as Local-RPC}. They showed that this local scheme gives better performance  than FIC~\cite{fic06} and other methods\ignore{ they compared with, in the time regimes as it operates}. However, this partitioning scheme obviously lacks consistency on the boundaries of the partitions and it was restricted to single-output GPR\ignore{, which limits its applicability in structured regression problems, e.g., human pose estimation}. Table~\ref{tab:thcomp} shows the complexity of this scheme denoted by local-RPC for GPR.

Beyond GPR, we found that local regression was adopted differently in structured regression models like Twin Gaussian Processes (TGP)~\cite{Bo:2010}, and also an data bias version of it, denoted by  IWTGP~\cite{Yamada:2012}.\ignore{ which deal with the data bias problem.In contrast to GPR, these models captures dependency between output dimensions, which leads to  better predictions in the Human 3D pose estimation task. } TGP and IWTGP outperform not only GPR in this task, but also various regression models including Hilbert Schmidt Independence Criterion (HSIC)~\cite{HSIC05}, Kernel Target Alignment (KTA)~\cite{KTA01}, and Weighted-KNN~\cite{Rasmussen:2005}. Both TGP and IWTGP have no closed-form expression for prediction. Hence, the prediction is made by gradient descent on a function that needs to compute the inverse of both the input and output kernel matrices, $O(N^3)$ complexity. Practically,  both approaches have been applied by finding the $M \ll N$ Nearest-Neighbors (NN) of each test  point in~\cite{Bo:2010} and~\cite{Yamada:2012}. The prediction of a test point is  $O(M^3)$ due to the inversion of $M \times M$ input and output kernel Matrices. \ignore{While this NN local regression scheme is one way to tackle the performance problem;} However, NN scheme has three drawbacks: (1) A regression model is computed for each test  point, which results in a scalability  problems in prediction (i.e., Matrix inversions on the NN of each each test point), (2) Number of neighbors might not be large enough to create an accurate prediction model since it is constrained by the first drawback, (3) It is inefficient compared with the other schemes used for GPR. Table~\ref{tab:thcomp} shows the complexity of this NN scheme\ignore{ denoted by NN}.
\ignore{
These drawbacks are overcome as a consequence of the six desirable properties. }

%------------------------------------------------------------------------- 
\section{ODC Framework}%
\label{sec:3}

\begin{figure*}
\centering
%\centering
%\begin{tabular}{cc}
%\bmvaHangBox{\fbox{
\includegraphics[width=1.0\textwidth]{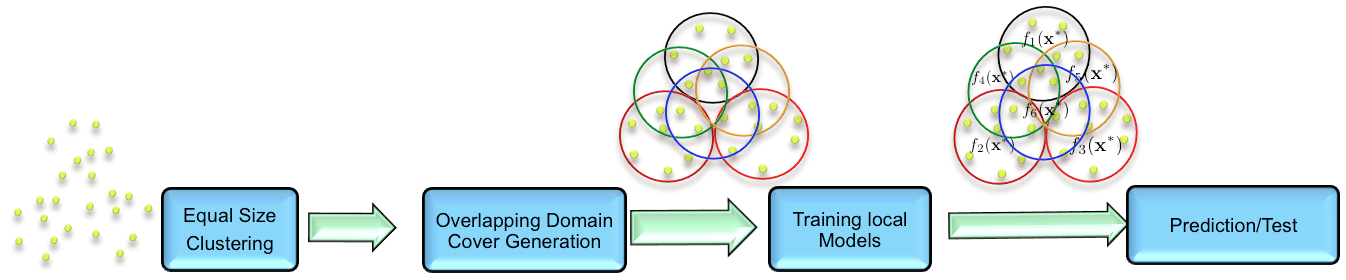}
\caption{ODC Framework}
\end{figure*}

The problems of the existing approaches, presented 
above, motivated us to develop an approach that  satisfies the properties listed in  table~\ref{tbl:contr}.\ignore{(a) accurate, (b) efficient, (c) scalable to arbitrary input dimension and also consider structured output, (d) consistent on boundary predictions (i.e., no discontinuity problem) (e) applicable to various kernel machines, and (f) easy to parallelize; } The table also shows which of these properties are satisfied for the relevant methods.\ignore{For example, looking at Park's approach \cite{park11},  it satisfies properties (a), (d), and
(f) but not properties (b) and (c).\ignore{ Property (b) is satisfied when the approach is applied to small dimensional input. Hence, property (c) is not satisfied.} Similarly, the likelihood approximation methods like FIC~\cite{fic06} and PIC~\cite{pic07}),  have properties (e) and (c) not satisfied, and (d) not applicable. Chalupka et al~\cite{Chalupka:2013}'s local method has properties (a), (c) and (d) not satisfied. Finally, the  NN-local regression scheme does not satisfy (b).}
 \ignore{
 GPR, TGP and IWTGP require inversion of $N \times N$ matrices, so the complexity of the solution is $O(N^3)$ which is impractical when $N$ is large. However, some approaches was designed to resolve this problem in GPR (e.g. FIC \cite{fic06}, PIC \cite{pic07}), that suffer from difficulty of parallelism. Furthermore, they were prominently addressed in GPR, which predict single dimension output and has been outperformed by structured regression approached like TGP. Practically, TGP \cite{Bo:2010} approach has been applied by finding the $N' << N$ nearest neighbors.  Then the prediction is done by minimizing equation ~\ref{eq:tgp} in $O(N'^3)$ operations (because of the inversion of $N' \times N'$ matrices). \ignore{$N'$ was chosen to be ($N' \approx 800$)}.  Similarity, IWTGP \cite{Yamada:2012} requires matrix inversions of $N' \times N'$ matrices, the complexity of solving equation ~\ref{eq:IWTGP}. IWTGP also includes the estimation of relative importance weight, thus the total complexity of IWTGP is  $O({N'}^3)+ O({N'}_{tst}^3)$. However computing  the nearest neighbors is one way to tackle the performance problem but it has three drawbacks. (1) A TGP /IWTGP model is computed for each query point $x$, which results in a scalability  problems in prediction (i.e. Matrix inversions on the nearest neighbor of each each test point) (2) Number of neighbors might not be large enough to create an accurate prediction model since it is constrained by the first drawback. These drawbacks are overcome as a consequence of the six desirable properties (see section ~\ref{sec:1}), which are  satisfied by our overlapping set cover framework. We denote this nearest neighbor variants of GPR, TGP and IWTGP as GPR-KNN, TGP-KNN, and IWTGP-KNN respectively.} In order to satisfy all the properties, we \ignore{\begin{wraptable}{r}{7.7cm} \vspace{-4mm}\caption{Contrast against most relevant methods}
 \vspace{3mm}
  \label{tbl:contr}
\scalebox{0.53}
{
\begin{tabular}{|l|c|c|c|c|}
\hline 
  &\cite{park11} & FIC/PIC~\cite{fic06}& NN~\cite{Bo:2010} & ODC \\ 
\hline 
Accurate & \begin{tabular}{@{}c@{}}No for  high  \\ input dimension \end{tabular}  & No & Yes& Yes \\ 
\hline 
Efficient & No & Yes & No &  Yes\\ 
\hline 
\begin{tabular}{@{}c@{}}Scalable to arbitrary \\ input dimension \end{tabular}  & No (2D) & Yes & Yes& Yes \\ 
\hline 
Consistent on Boundaries & Yes & No & Yes & Yes \\ 
\hline 
supported kernel machines &GPR& GPR & TGP & GPR, TGP, IWTGP and others\\ 
\hline 
Easy to parallelize & No & No & Yes& Yes \\ 
\hline 
\end{tabular} 
}\vspace{-3mm}\end{wraptable}} present the Overlapping Domain Cover (ODC) notion. We define the ODC as a collection of overlapping subsets of the training points, denoted by  subdomains,  such that they are as spatially coherent as possible.\ignore{In order to achieve this objective, we model our framework as follows} During training, an ODC is computed such that each subdomain overlaps with the neighboring subdomains. Then, a  local prediction model (kernel machine) is \ignore{ \begin{wrapfigure}{r}{0.45\textwidth}
  \vspace{-12pt}
    \includegraphics[width=0.45\textwidth]{systemArch3.png}
  \caption{Overlapping Domain Cover Framework, applied to 3D pose instances taken from Human Eva dataset \cite{SigalBB10} }
  \label{fig:DDTGP}
%\label{fig:teaser}
  \vspace{-15pt}
\end{wrapfigure}} created for each subdomain and the computations that does not depend on the test data are factored out and precomputed (e.g. inversion of matrices). The nature of the ODC generation makes these kernel machines consistent in the overlapped regions, which are the boundaries since we constraint the subdomains to be coherent. This is motivated by the notion that data lives on a manifold with local properties and consistent connections between its neighboring regions. On prediction, the output is calculated as a reduction function of the predictions on the closed subdomain(s). Table~\ref{tab:thcomp}\ignore{, we show the computational complexity on testing and training including clustering if applicable using our parametrization and notations, i.e., as a function of $M$, $N$, $d_X$, and $d_Y$. The first row  shows the computational complexity for using full data and NN local regression scheme,  applied to either TGP or GPR. The third and the fourth rows show the complexity of FIC and Local-RPC applied to GPR in ~ and  ~\cite{Chalupka:2013} respectively.} ( the last row) shows the complexity for our generalized ODC framework, detailed in Sec~\ref{sec:4} and ~\ref{sec:pred}. In contrast to the prior work, our ODC framework is designed to cover structured regression setting, $d_Y>1$ and to be applicable to GPR, TGP, and many other models.\ignore{Having computed the ODC, a local kernel machine is trained on each member of the cover, denoted as subdomains. At test time, the  local kernel machines of the closest subdomains to the test point are retrieved and the final output is computed as a combination of their predictions.}

\begin{table}[t!] 
\caption{Contrast against most relevant methods}
  \label{tbl:contr}
\scalebox{0.68}
{
\begin{tabular}{|l|c|c|c|c|}
\hline 
  &\cite{park11} & FIC/PIC~\cite{fic06}& NN~\cite{Bo:2010} & ODC \\ 
\hline 
Accurate & \begin{tabular}{@{}c@{}}No for  high  \\ input dimension \end{tabular}  & No & Yes& Yes \\ 
\hline 
Efficient & No & Yes & No &  Yes\\ 
\hline 
\begin{tabular}{@{}c@{}}Scalable to arbitrary \\ input dimension \end{tabular}  & No (2D) & Yes & Yes& Yes \\ 
\hline 
Consistent on Boundaries & Yes & No & Yes & Yes \\ 
\hline 
supported kernel machines &GPR& GPR & TGP & GPR, TGP, IWTGP and others\\ 
\hline 
Easy to parallelize & No & No & Yes& Yes \\ 
\hline 
\end{tabular} 
}\end{table}
\ignore{
The main intuition behind the ODC framework is to } \ignore{ create subdomains with optimized  consistency of prediction on the boundaries and }\ignore{increase the likelihood of finding a local kernel machine trained in a sufficient neighborhood of an arbitrary test point. {The framework is also motivated by the notion that \ignore behind this procedure is that }Data lives on a manifold, which has local properties and consistent connections between its neighboring subdomains. As a result, our framework is to model the problem as finding/learning these subdomains such that they have consistent predictions as possible, which is achieved by building overlapping kernel machines}\ignore{; Figure ~\ref{fig:DDTGP} shows the block diagram of our method, applied to 3D pose estimation as a regression problem.} 

\noindent\textbf{Notations.} Given a set of input data  $X=\{\mathbf{x}_1, \cdots, \mathbf{x}_N  \}$, our prediction framework firstly  generates a set of non-overlapping equal-size partitions,  $C=\{C_1, \cdots, C_K\}$, such that $\cup_{i}{C}_i = X$, $| C_i| = N/K$. Then, the ODC is defined based on them as $\mathcal{D} =  \{ {D}_1, \cdots, {D}_K \}$, such that $|{D}_i| =M\, \forall i$, ${D}_i = C_i \cup O_i, \forall i$\ignore{ is defined based on them, such that  }. $O_i$ a the set of points that overlaps with the other partitions, i.e.,  $O_i  = \{x: x \in  \{\cup_{j \neq i} {C}_j\} \}$,  such that $|O_i| = p \cdot M$, $|C_i| = (1-p) \cdot M$,   $0 \le p \le 1$ is the ratio of points in each overlapping subdomain, $D_i$, that belongs to/overlaps with partitions, other than its own, $C_i$. 

It is important to note that, the ODC could be specified by two parameters, $M$  and $p$, which are the number of points in each subdomain and the ratio of overlap respectively; this is since $K= N/ (1-p) M$.  This parameterization of ODC generation is reasonable for the following reasons. First, $M$ defines the number of points that are used to train each local kernel machine, which controls the performance of the local prediction. Second, given $M$ and that $K= N/ (1-p) M$, $p$  defines how coarse/fine the distribution of kernel machines are.  It is not hard to see that as $p$ goes to $0$, the generated ODC reduces to the set of non-overlapping clusters. Similarly, as $p$ approaches $1 - 1/M$, the ODC reduces to generating a cluster at each point with maximum overlap with other clusters, i.e., $K=N$, $|C_i|=1$, and $|O_i| = M -1$. 
\textit{Our main claim is two fold. First, precomputing local kernel machines (e.g.  GPR, TGP, IWTGP) during training on the ODC significantly increase the speedup on prediction time. Second, given a fixed $M$ and $N$, as $p$ increases, local prediction performance increases, theoretically supported by Lemma~\ref{lemma1}} 
\vspace{-0.5mm}
\begin{lem}
Under ODC notion, as the overlap $p$ increases, the  closer the nearest model to an arbitrary test point and the more likely that model get trained on a big neighborhood of the test point.  
\label{lemma1}
\end{lem}

\begin{proof}
We start by outlining the main idea behind  the proof, which is directly connected to the fact that $K= N/ (1-p) M$, which indicates that the number of local models increases as $p$ increases given fixed $N$ and $M$.  Under the assumption that the local models are spatially cohesive,  $p \to 1$ theoretically indicates that there is a local model  centered at each point in the space (\ie $K = \infty$). Hence, as $p$ increases, the distribution of the kernel machines is the finest and the more likely a test point to find the closest kernel machines trained on a big neighborhood of it leading to more accurate prediction. Meanwhile, as $p$ goes to $0$, the distribution is the coarsest and the less likely a test point finds, the closest kernel machines, trained on a big neighborhood. 

Let's assume that each kernel machine is defined on $M$ points that are spatially cohesive, covering the space of N points with $\frac{N}{(1-p)M}$. Let's assume that center of the $M$ points in kernel machine $i$ is $\mu_i$, the the Co-variance matrix of these points are $\Sigma_i$. Hence  
\begin{equation}
\begin{split}
p(\textbf{x}|D_i) &=  \mathcal{N}(\mu_i, \Sigma_i) \\
&=  (2 \pi)^{-\frac{d_X}{2}} |\Sigma_i|^{-\frac{1}{2}} e^{-\frac{1}{2} (\textbf{x}-\mu_i)^\mathsf{T} \Sigma_i^{-1}
 (\textbf{x}-\mu_i)}  
\end{split}
\end{equation}
where $\mathcal{N}(\mu_i, \Sigma_i)$ is a normal distribution of mean  $\mu_i$ and Co-variance matrix $\Sigma_i$.

Let's assume that there are two ODCs, $ODC_1$ and $ODC_2$, defined on the same $N$ points, the first one has  overlap $p_1$ and the second one is with overlap $p_2$, such that, $p_2>p_1$. Let's assume that the number  of kernel machines in $ODC_1$ and $ODC_2$ are  $K_1$ and $K_2$, respectively. Hence, 
\begin{equation}
\begin{split}
K_1 = \frac{N}{(1-p_1)M}\,\,\,, \,\,\, K_2 = \frac{N}{(1-p_2)M} 
\end{split}
\end{equation}

Since $p_2>p_1$, $0\le p_1 < 1$ and $0\le p_2 < 1$,  then $K_2>K_1$, which indicates that the number of kernel machines in $ODC_2$ with higher overlap is bigger than the number of kernel machines in $ODC_2$.   Let's assume that there is an  test point $\textbf{x}^*$ and define that the probability that $\textbf{x}^*$ is captured by the ODC to be proportional to the maximum probability of $\textbf{x}^*$ among the domains. 

\begin{equation}
\begin{split}
p(\textbf{x}^*) &  = \sum_{i=1}^{K}   p(\textbf{x}^*,D_i) \\ & = \sum_{i=1}^{K}   p(\textbf{x}^*|D_i) \delta ( p(\textbf{x}^*|D_i) -  max_{j=1}^K (p(\textbf{x}^*|D_i))) \\
  & =   max_{i=1}^{K} p(\textbf{x}^*|D_i) 
\\ 
& = (2 \pi)^{-\frac{d_X}{2}} max_{i=1}^{K}  |\Sigma_i|^{-\frac{1}{2}} e^{-\frac{1}{2} (\textbf{x}^*-\mu_i)^\mathsf{T} \Sigma_i^{-1}
 (\textbf{x}^*-\mu_i)}
\end{split}
\end{equation}
where $\delta(0)=1$, $0$ otherwise\ignore{ is the Dirac delta function}.  The reason behind this definition of  $p(x^*)$ is that our method select the domain of preduction based on $arg max_{i=1}^{K} p(\textbf{x}^*|D_i)$.  
Hence $p_{ODC_1}(x^*) = max_{i=1}^{K_1} p_{ODC_1}(\textbf{x}^*|D_i)$ and $p_{ODC_2}(x^*) = max_{i=1}^{K_2} p_{ODC_2}(\textbf{x}^*|D_i)$. 

We start by the case where the points are uniformally distributed in the space. Under this condition  and assuming that spatially cohesive domain cover, this leads to that $p(\textbf{x}^*|D_i) \approx  \mathcal{N}(\mu_i, \Sigma) \forall i$, where $\Sigma_1 = \Sigma_2 \cdots =\Sigma_K = \Sigma$.  Hence 

\begin{equation}
\begin{split}
p(\textbf{x}^*|D_i) \propto  e^{-\frac{1}{2} (\textbf{x}^*-\mu_i)^\mathsf{T} \Sigma^{-1}
 (\textbf{x}^*-\mu_i)}  \\
ln(p(\textbf{x}^*|D_i)) \propto  - (\textbf{x}^*-\mu_i)^\mathsf{T} \Sigma^{-1}
 (\textbf{x}^*-\mu_i)
\end{split}
\end{equation}
Then 
\begin{equation}
\begin{split}
p(\textbf{x}^*) & =   max_{i=1}^{K} p(\textbf{x}^*|D_i) 
\\ 
& = (2 \pi)^{-\frac{d_X}{2}} \Sigma|^{-\frac{1}{2}} max_{i=1}^{K}  | e^{-\frac{1}{2} (\textbf{x}-\mu_i)^\mathsf{T} \Sigma^{-1}
 (\textbf{x}-\mu_i)} \\
 & \propto  max_{i=1}^{K}  e^{-\frac{1}{2} (\textbf{x}-\mu_i)^\mathsf{T} \Sigma^{-1} (\textbf{x}-\mu_i)} \\
ln( p(\textbf{x}^*)) & \propto max_{i=1}^{K}-  (\textbf{x}-\mu_i)^\mathsf{T} \Sigma^{-1}(\textbf{x}-\mu_i)
\end{split}
\end{equation}
Hence,  $p(\textbf{x}^*)$ gets maximized as it get closer to one of the centers of the domains $\mu_i$, defined by the ODC.  It is not hard to seen that  that chances of $\textbf{x}^*$ to be closer to one of the centers covered by $ODC_2$ is higher than  $ODC_2$, especially when $p_2 \gg p_1$.  This is since $K_1 = \frac{N}{(1-p_1)M}, K_2 = \frac{N}{(1-p_2)M}$. Hence $K_2 \gg K_1$ when $p_2 \gg p_1$. For instance, when $p_1=0$ and  $p2=0.9$, this leads to that $ODC_1$ will generate $K_1 = \frac{N}{M}$ domains, while $ODC_2$ will generate    $K_2= \frac{10 \cdot N}{M} = 10 K_1$, which is ten times more domains and centers. The fact that there are much more domains if $K_2 \gg K_1$ together with that there domains are spatially cohesive leads to $max_{i=1}^{K_1}-  (\textbf{x}^*-\mu^1_i)^\mathsf{T} \Sigma_1^{-1}(\textbf{x}^*-\mu^1_i) \gg max_{i=1}^{K_2}-  (\textbf{x}^*-\mu^2_i)^\mathsf{T} \Sigma_2^{-1}(\textbf{x}^*-\mu^2_i)$. The proof of this statement derives from the fact that $max_{i=1}^{K}   - (\textbf{x}^*-\mu_i)^\mathsf{T} \Sigma^{-1}(\textbf{x}^*-\mu_i)$ is could maximized by (1) if   $\textbf{x}^*$ gets very close to one of  $\mu_i, i=1:K$,and (2) smaller variance $|\Sigma|$, which is minimized by the nature by which ODC is created, since each domain $i$ is created by neighboring points to its center (\ie $|\Sigma_1| \gg |\Sigma_2|$). This directly leads to that if $K_2 \gg K_1$ then  
$max_{i=1}^{K_1}-  (\textbf{x}^*-\mu^1_i)^\mathsf{T} \Sigma_1^{-1}(\textbf{x}^*-\mu^1_i) \gg max_{i=1}^{K_2}-  (\textbf{x}^*-\mu^2_i)^\mathsf{T} \Sigma_2^{-1}(\textbf{x}^*-\mu^2_i)$. Hence,  $p_{ODC_2}(x^*) \gg p_{ODC_1}(x^*)$.

Even if the points are not uniformally distributed, it is still more likely that an ODC with higher overlap would have higher $p(x^*)$, since $x^*$ is close under expectation to one of the centers if more spatially cohesive domains are generated which increases with higher overlap. Our experiments also proves that the ODC concept generalizes on three real dataset where the training points are not distributed uniformally.  
\end{proof}

%\begin{proof}
%This proof is directly connected to the fact that $K= N/ (1-p) M$, which indicates that the number of local models increases as $p$ increases given fixed $N$ and $M$.  Under the assumption that the local models are spatially cohesive,  $p \to 1$ theoretically indicates that there is a local model  centered at each point in the space (i.e., $K = \infty$). Hence, as $p$ increases, the distribution of the kernel machines is the finest and the more likely a test point to find the closest kernel machines trained on a big neighborhood of it leading to more accurate prediction. Meanwhile, as $p$ goes to $0$, the distribution is the coarsest and the less likely a test point finds, the closest kernel machines, trained on a big neighborhood \ignore{\footnote{Neighborhood is a subset of the training data}} of it, which might lead to inaccurate prediction.
%\end{proof}
\vspace{-2mm}
\ignore{In order to better model the local prediction on each subdomain, we aim at creating a spatially cohesive ODC as possible, detailed in Sec~\ref{sec:4}. }

\ignore{\begin{figure}[h!]
\centering
\includegraphics[width=0.45\textwidth]{systemArch.png}
  \caption{[TO BE UPDATED OR REMOVED] Overlapping Domain Cover applied to 3D-pose estimation with pose instances taken from Human Eva dataset  \cite{SigalBB10} }
  \label{fig:DDTGP}
\end{figure}}

\begin{comment}
In order to make this notion more concrete, we present the prediction framework in Figure ~\ref{fig:DDTGP}.}

 Training is done by clustering the data and creating consistent subdomains from this data. Then a kernel machine (e.g. TGP, GPR, IWTGP) is created for each of these subdomains.  We denote cluster information and the local kernel machines for the subdomains as ODC Model. During prediction, closest clusters for the test data are retrieved\ignore{ based on different modes ( detailed in next sections \ignore{~\ref{sec:pred})}}. Afterwards ( as an optional phase), the submodels are weighted to handle data bias of the training data. We denote $CovFlag$ as a control flag that enables this phase \footnote{Note that IWTGP is equivalent to TGP if $W  = I$, $I$ is the identity matrix}. Eventually, the final predictions are made by combining the outputs of these subdomain models.

Our framework configuration starts by specifying two parameters $M$ and the ratio of overlapping points for each subdomain $p$, such that 0$ \le p \le 1$. For each subdomain  $p M$ of the points overlap with other subdomains. 
\end{comment}

\subsection{Training}
\label{sec:4}
There are several overlapping clustering methods that include (e.g.~\cite{perez2013oclustr} and~\cite{bonchi2013overlapping}), which looks relevant for our framework. However these methods does not fit our purpose both equal-size  constraints for the local kernel machines. We also found them very slow in practice because their complexity varies from cubic to quadratic (with a big constant factor) on the training-set. These problems motivated us to propose a practical method that builds overlapping local kernel-machines with spatial and equal-size constraints. \ignore{Efficiently decomposing the domain with equal-size constraints is critical for our purpose. This is} These constraints are critical for our purpose since the number of points in each kernel-machine determine its local performance. Hence,  our training phase is two steps: (1) the training data is split into $K= N/ (1-p) M$ equal-sized clusters of $(1-p) M$ points. (2) an ODC with K overlapping subdomains is generated by augmenting each cluster with $p \cdot M$ points from the neighboring clusters.\ignore{ based on $M$ and $p$ and an ODC training model is learned by building a kernel machine on each subdomain.}
%\begin{figure}[t!]
%\centering
%    \begin{subfigure}[b]{0.20\textwidth}
%    \includegraphics[width=1.0\textwidth,height=0.75\textwidth]{Ekmeans57.png}
%               \caption{$\,\,$}
%                \label{fig:ekmeans}
%        \end{subfigure}%
%        \begin{subfigure}[b]{0.22\textwidth}
%   \includegraphics[width=1.0\textwidth,height=0.71\textwidth]{figMInvSpeedUp.eps}
%                \caption{$\,\,$}
%                 \label{fig:SpeedUp}
%        \end{subfigure}
%        \vspace{-4mm} 
%        \caption{(a) AB-EKmeans on 300,000 2D points, K= 57\ignore{ (best seen electronically)}(b) ODC-prediction speedup on either TGP or GPR for retrieving precomputed matrix inverses as $M$ increases, compared with computing them on test time by NN scheme (log-log scale)}
%\vspace{-4.5mm}        
%\end{figure}
%\begin{figure}
%\centering
%\begin{tabular}{cc}
%\bmvaHangBox{\fbox{
%\includegraphics[width=2.8cm]{Ekmeans57.png}}}&
%\bmvaHangBox{\fbox{\includegraphics[width=2.8cm]{figMInvSpeedUp.eps}}}\\
%%\bmvaHangBox{\fbox{\includegraphics[width=5.6cm]{OverlappingFig.png}}}\\
%(a)&(b)
%\end{tabular}
%\caption{(a) AB-EKmeans on 300,000 2D points, K= 57\ignore{ (best seen electronically)}(b) ODC-prediction speedup on either TGP or GPR for retrieving precomputed matrix inverses as $M$ increases, compared with computing them on test time by NN scheme (log-log scale)}
%\label{fig:overlapandnonoverlap}
%\end{figure}

\vspace{-1.5mm}
\subsubsection{Equal-size Clustering}

\ignore{The performance of the local regression model depends on the number of points in each partition.
In order to bound the computation, we need the subdomains to be balanced in size. This is achieved by splitting the data into equal-sized clusters, and use these clusters to establish the subdomains. } There are recent algorithms that deal with size constraints in clustering. For example,\ignore{Zhu \etal}~\cite{Zhu2010} formulated the problem of clustering with size constraints as a linear programming problem. However such algorithms are not computationally efficient, especially for large scale datasets (e.g., Human3.6M). We study two efficient ways to generate equal size clusters; see Table~\ref{tab:thcomp} (last row) for their ODC-complexity.

\vspace{2mm}

\noindent \textbf{Recursive Projection Clustering (RPC) \cite{Chalupka:2013}.}
In this method,\ignore{\cite{Chalupka:2013} proposed to partition} the training data is partitioned to perform GPR prediction. \ignore{In summary, they begin with assigning } Initially all  data points are put in one cluster. Then, two points are chosen randomly\ignore{ to create a line through them.} and orthogonal projection of all the data onto the line connecting them is computed. Depending on the median value of the projections, The data is then split into two equal size subsets. The same process is then applied to each cluster to generate $2^l$ clusters after $l$ repetitions. The iterations stops once $2^l>K$. As indicated, the number of clusters in this method has to be a power of two and it might produce long thin clusters. 

\ignore{K-means clustering ~\cite{Hartigan1979} splits the domain such that the clusters have similar spatial extent, resulting in Voronoi partitioning.  However,  this is not what is needed for efficient computation of kernel machines.}

\ignore{
Therefore, we propose two greedy algorithms to obtained equal-size clustering. }

\vspace{2mm}

\noindent \textbf{Equal-Size K-means (EKmeans).} We propose a variant of k-means clustering~\cite{Hartigan1979} to generate equal-size clusters.\ignore{The algorithms mainly modify the assignment step of the k-means to bound the size of the resulting clusters.  }\ignore{Given a set of data points $X=\{\mathbf{x}_1, \cdots, \mathbf{x}_N  \}$, t} The goal is to obtain disjoint partitioning of $X$ into clusters $C=\{C_1, \cdots, C_K\}$,  similar to the k-means objective, minimizing the within-cluster sum of squared Euclidean distances, $ C = \arg_C J(C) = \min \sum_{j=1}^K \sum_{\mathbf{x_i}\in C_j} d(\mathbf{x_i}, \boldsymbol{\mu}_j)$,  
\ignore{\[ \small C = \arg_C \min \sum_{j=1}^K \sum_{\mathbf{x_i }\in C_j} d(\mathbf{x_i},\boldsymbol{\mu}_j) \]}
where $\boldsymbol{\mu}_i$ is the mean of cluster $C_i$, and $d(\cdot,\cdot)$ is the squared distance. Optimizing this objective is NP-hard and  k-means iterates between the assignment and update steps as a heuristic to achieve a solution; $l_1$ denotes number of iterations of kmeans. We add equal-size constraints  $\forall  (1 \leq i  \leq K),  |C_i| = N/K = (1-p) M$. \ignore{We add extra constraints to the problem to guarantee equal-sized partitioning $\forall  (1 \leq i  \leq K),  |C_i| = N/K$. }

\ignore{
In order to achieve this partitioning, we propose two heuristic algorithms, which mainly modify the assignment step of the k-means to bound the size of the resulting clusters. We introduce here one of them which achieves less cost $J(C)$; see SM for pseudo code of the two algorithms and our cost analysis of them. We call the better algorithm {\em Assign and Balance (AB) EKmeans}, in which }
\ignore{{\em (2) Iterative Minimum-Distance Assignments (IMDA) k-means:} we initialize a pool of unassigned points $\tilde{X}  =  X$ and initialize all clusters as empty.  Given the means computed from the previous update steps, we compute the distances $d(\mathbf{x}_i,\boldsymbol{\mu}_j)$ for all points/center pairs. We iteratively pick the minimum distance pair 
\[\small (\mathbf{x}_p,\mathbf{mu}_l)  : d(\mathbf{x}_p,\boldsymbol{\mu}_l) \le d(\mathbf{x}_i,\boldsymbol{\mu}_j)  \forall x_i \in \tilde{X} \text{and}  |C_l| < N/K \]
and assign point $x_p$ to cluster $l$. The point is then removed from the pool of unassigned points. if  $|C_l| = N/K$,  then it is marked as balanced and no longer considered. The process is repeated until the pool is empty.{\em (3) Assign and Balance (AB) k-means:}}

In order to achieve this partitioning, we propose an efficient heuristic algorithm, denoted by {\em Assign and Balance (AB) EKmeans}. It mainly modifies the assignment step of the k-means to bound the size of the resulting clusters. We first assign the points to their closest see center as typically done in the assignment step of k-means. We use $C(\mathbf{x}_p)$ to denote the cluster assignment of a given point $\mathbf{x}_p$. This results in three types of clusters: balanced, overfull, and underfull clusters. Then some of  the points in the overfull clusters are redistributed to the underfull clusters by assigning each of these points to the closest underfull cluster.  This is achieved by initializing a pool of overfull points defined as $ \tilde{X}  = \{\mathbf{x}_p : \mathbf{x}_p \in C_l , |C_l| > N/K \}$; see Figure~\ref{figABkmeans57}.\ignore{ \begin{wrapfigure}{r}{0.45\textwidth}
  \vspace{-12pt}
\begin{tabular}{ccc}
\bmvaHangBox{\fbox{\includegraphics[width=2.3cm]{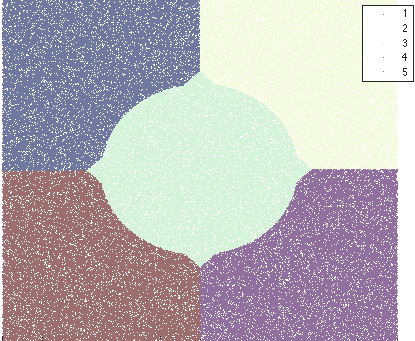}}}&
\bmvaHangBox{\fbox{\includegraphics[width=2.7cm]{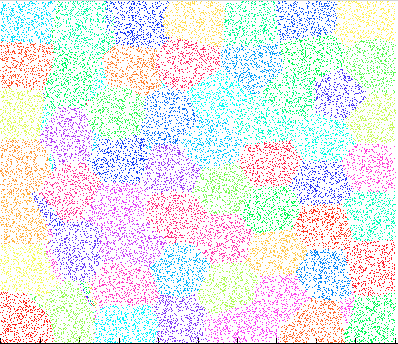}}}\\
(a)&(b)
\end{tabular}
\caption{Applying our Assign and Balance variant of Kmeans on 300,000 random 2D points: (a) 5 clusters; (b)  57 clusters (best seen in color).}
\label{fig:ekmeans}
%\label{fig:teaser}
\end{wrapfigure}}

\begin{figure}[t!]
\includegraphics[width=0.45\textwidth]{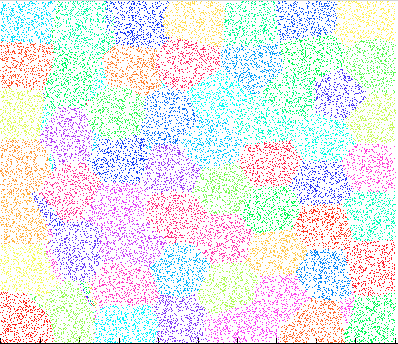}
\caption{AB-EKmeans on 300,000 2D points, K= 57}
\label{figABkmeans57}
\end{figure}

Let us denote the set of underfull clusters by $\tilde{C} = \{C_p: |C_p| < N/K \} $. We compute the distances $d(\mathbf{x}_i, \boldsymbol{\mu}_j), \forall \mathbf{x_i} \in \tilde{X} \, \text{and}\, C_i \in \tilde{C}$. Iteratively, we pick the minimum distance pair $(\mathbf{x}_p,\boldsymbol{\mu}_l) $ and assign $\mathbf{x}_p$ to cluster $C_l$ instead of cluster $C(\mathbf{x}_p)$. The point is then removed from the overfull pool. Once an underfull cluster becomes full it is removed from the underfull pool, once an overfull cluster is balanced, the remaining points of that cluster are removed from overfull pool. The intuition behind this algorithms is that,  the cost associated with the initial optimal assignment (given the computed means) is minimally increased by each swap since we pick the minimum distance pair in each iteration. Hence the cost is kept as low as possible while balancing the clusters. We denote the the name of this Algoirthm as Assign and Balance EKmeans. Algorithm~\ref{alg:ddclusterALg1} illustrates the overall assignment step and Fig.~\ref{fig_ABKmeans_balancing} visualizes the balancing step.

\begin{figure} [t!]
\centering
%\centering
%\begin{tabular}{cc}
%\bmvaHangBox{\fbox{
\includegraphics[width=0.40\textwidth]{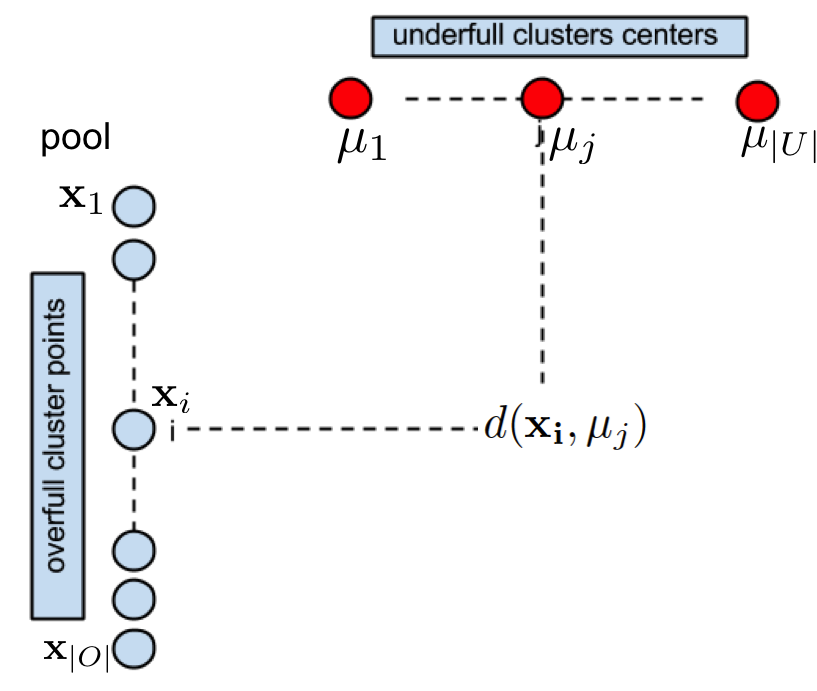}
\caption{AB Kmeans: Balancing Step}
\label{fig_ABKmeans_balancing}
%\vspace{-3mm}
\end{figure}

\begin{algorithm}[t!]
\KwIn{$\textbf{X} (N \times d_x ),{\{\boldsymbol{\mu}_i\}}_{i=1}^{K}$}
\KwOut{labels}
1- Assign the points initially to its closest center; this will put the clusters into 3 groups (1) balanced clusters   (2) overflowed clusters (3) under-flowed clusters.\\
2- Create a matrix $\textbf{D} \in R^{N \times K}$, where $\textbf{D}[i,j]$ is the distance between the $i^{th}$ point to the $j^{th}$ cluster center; rows are restricted points belongs only to the overflowed clusters; columns are restricted to underflowed cluster centers \\
3- Get the coordinate $(i_*,j_*)$ that maps the smallest distance in $\textbf{D}$.\\
4- Remove the $i_*^{th}$ row from matrix $\textbf{D}$ and mark it as assigned to the $j^{th}$ cluster\\
5- If the size of the cluster $j$ achieves the ideal size (i.e. ~ $n/K$), then remove the $j^{th}$ column from matrix $\textbf{D}$.\\
6- Go to step 3 if there is still unassigned points
\caption{Assign and Balance (AB) k-means: Assignment Step}
\label{alg:ddclusterALg1}
\end{algorithm}

\ignore{
We performed experiments on synthetic datasets (see section ~\ref{sec:6}) on the two proposed variants of k-means algorithms and found that the Assign and Balance algorithm outperform the Iterative Assignment algorithm. Therefore, we focused on RPC and the Assign and Balance algorithm as a pre-processing step for generating the overlapping subdomains. Figure ~\ref{fig:ekmeans} shows example output of the  Assign and Balance algorithm  for assignment on random 2D points. We attach the pseudo-code of the two algorithms in the supplementary materials.}

\vspace{-1.5mm}
\subsubsection{Overlapping Domain Cover(ODC) Model}
\label{sec:train1}

 Having generated the disjoint equal size clusters, we generate the ODC subdomains based on the overlapping ratio $p$, such that $p \cdot M$ points are selected from the neighboring clusters. Let's assume that we select only the closest $r$  clusters to each cluster,  $C_i$ is closer to $C_j$ than $C_k$ if $\| \boldsymbol{\mu}_i - \boldsymbol{\mu}_j\| < \| \boldsymbol{\mu}_i - \boldsymbol{\mu}_k\|$. It is important to note that $r$ must be greater than  $p/(1-p)$ in order to supply the required $p \cdot M$ points; this is since number of points in each cluster is $(1-p) M$. Hence, the minimum value for $r$ is $\lceil(p \cdot M)/ ( (1-p)  \cdot M)\rceil = \lceil p/(1-p) \rceil$ clusters. Hence, we parametrize $r$ as $r =  \lceil t \cdot p/(1-p) \rceil, t\ge 1$. We study the effect of $t$ in the experimental results section. Having computed $r$ from $p$ and $t$, each subdomain $D_i$ is then created by merging the points in the cluster $C_i$ with $p \cdot M$ points, retrieved from the $r$ neighboring clusters. Specifically, the points are selected by sorting the points in each of $r$ clusters by the distance to $\boldsymbol{\mu}_i$. The number of points retrieved for each of the $r$ neighboring clusters is inversely proportional to the distance of its center to $\boldsymbol{\mu}_i$. If a subset of the $r$ clusters are requested to retrieve more than its capacity (i.e., $(1-p) M$), the set of the extra points are requested from the remaining clusters giving priority to the closer clusters (i.e., starting from the nearest neighboring cluster to the cluster on which the subdomain is created). \ignore{Its is worth to mention that a}As $t=1$ and $p$ increases, all points that belong to the $r$ clusters tends to be merged with $C_i$. In our framework, we used FLANN ~\cite{flann09} for fast NN-retrieval; see pseudo-code of ODC generation in  Appendix C.
 \ignore{
 ????? Does it need a clarification or an example ????.  
 We constraint these points to belong to up to $r$ clusters. If they belong to more than $r$ clusters, The closest clusters are chosen based on the distance between the cluster centers (i.e.  $\|\ \boldsymbol{\mu}_i - \boldsymbol{\mu}_j\|\, j=1:K , j \neq i$ ). Since, $|C_i| \approx N / K$ for equal size clusters, we choose $r \ge b \cdot K / N$, to guarantee that that there is at least $b$ points in the neighboring $r$ clusters. Hence, we define number of points in each subdomain as $M \approx N / K + b$. A successful assignment from our experiments is to increase $b$ with respect to $N/K$. We attach a detailed prescription of the parameter selection and the pseudo-code of the our Overlapping Domain Cover algorithm  and implementation details in the supplementary materials.}
 
\begin{comment}
\begin{algorithm}
{\textbf{Input:} Clusters ${\{C_k\}}_{k=1}^{K} $}
\KwOut{Overlapping subdomains ${\{D_k\}}_{k=1}^{K}$}
\ForEach{Cluster $C_k$}{
Compute the closest $OCC$ clusters ${\{{{OVC_K}_i}\}}_{i=1}^{OCC}$ based on $DK_i = \| \boldsymbol{\mu}_k- \boldsymbol{\mu}_i \|$ , $i\neq k$\\
Let $LK_i = 1/DK_i,  {WK_i} =  \frac{LK_i}{\sum_{l=1}^{OCC} LK_l}$  ${i=1 : OCC}$\\
Let ${NPK_i} =  floor(WK_i * OPC)$, ${i=1 : OCC}$ \\
Let $ExKPts = OPC - \sum_{l=1}^{OCC} NPK_l$ \\
Let ${NPK_i}$ =  $NPK_i +1$ , $ i=1 : ExKPts $\\
$D_k =  C_k$ \\
\For{i=1 : OCC} { $Ps_i = KNN({OVC_K}_j,NPK_i )$ \\ $D_k = D_k \cup Ps_i$ }

\Comment{where KNN is the K-nearest neighbors algorithms. For high performance calculation of $KNN$, we use FLANN \cite{flann09} to calculate $KNN$.}
}
\caption{Subdomains Generation (Note: All ${\{D_k\}}_{k=1}^{K}$ are stored as indices to $X$).  }
\label{alg:sdgen}
\end{algorithm}
\end{comment}

%\vspace{-1mm}
%\subsection{ODCModel Generation}
%\label{sec:train2}

\ignore{Having generated the ODC,}After the ODC is generated, we compute the the sample normal distribution using the points that belong to each subdomain. Then, a local kernel machine is trained for each of the overlapping subdomains. We denote the point set normal distribution of the subdomains as $p(\mathbf{x}|D_i) = \mathcal{N}(\boldsymbol{\mu}'_i \in R^{d_X}, \Sigma'_i \in R^{d_X \times d_X})$;\ignore{
\begin{equation}
p(\mathbf{x}|D_i) = \mathcal{N}(\boldsymbol{\mu}'_i \in R^{d_x}, \Sigma'_i \in R^{d_x \times d_x})
\end{equation}}
 ${\Sigma'}_i^{-1}$ is precomputed during the training for later use during the prediction. \ignore{The main rationale behind using our overlapping subdomain notion, is} Finally, we factor out all the computations that does not depend on the test point (for GPR, TGP, IWTGP) and store them with each sub domain as its local kernel machine. We denote the training model for subdomain $i$ as $\mathcal{M}^i$, which is computed as follows for GPR and TGP respectively.%; see SM for IWTGP.
\ignore{
foot  IWTGP respectively.}

\textbf{ {GPR}.}\ignore{In order to compute $\mathcal{M}^i$ In the case of GPR, we first precompute $(\textbf{K}^i_j + \sigma^2_{n^i_j} \textbf{I})^{-1} $, where  $\textbf{K}^i_j$ is an $M \times M$  kernel matrix, defined on the input points that belong to domain $i$. Since, GPR does not capture dependency between output dimensions, each dimension $j$ in the output could have its own hyper-parameters, which results in a different kernel matrix for each dimension $\textbf{K}^i_j$. We also precompute $(\textbf{K}^i_j + \sigma^2_{n^i_j} \textbf{I})^{-1} \textbf{y}_j$ for each dimension. Hence  $\mathcal{M}^i = \{ ( \textbf{K}^i_j + \sigma^2_{n^i_j} \textbf{I})^{-1},  (\textbf{K}^i_j + \sigma^2_{n^i_j} \textbf{I})^{-1} \textbf{y}_j ), j=1 : d_Y \}$.} Firstly, we  precompute $(\textbf{K}^i_j + \sigma^2_{n^i_j} \textbf{I})^{-1} $, where  $\textbf{K}^i_j\,$ is an $M \times M\,$  kernel matrix, defined on the input points in  $D_i$. Each dimension $j$ in the output could have its own hyper-parameters, which results in a different kernel matrix for each dimension $\textbf{K}^i_j$. We also precompute $(\textbf{K}^i_j + \sigma^2_{n^i_j} \textbf{I})^{-1} \textbf{y}_j$ for each dimension. Hence  $\mathcal{M}^i_{GPR} = \small\{ ( \textbf{K}^i_j + \sigma^2_{n^i_j} \textbf{I})^{-1},  (\textbf{K}^i_j + \sigma^2_{n^i_j} \textbf{I})^{-1} \textbf{y}_j )$ $, j=1 : d_Y \}\normalsize$.

\textbf{{TGP.}} \ignore{Since, TGP captures the the dependency between the outputs, t } The local kernel machine for each subdomain in TGP case is defined as $\mathcal{M}^i_{TGP} = \small\{(\textbf{K}_X^i + \lambda_{X}^i  \textbf{I})^{-1}, (\textbf{K}_Y^i + \lambda_{Y}^i  \textbf{I})^{-1}\}\normalsize$,  where $\textbf{K}_X^i$ and $\textbf{K}_Y^i$ are $M \times M$ kernel matrices defined on the input points and the corresponding output points respectively, which belong to domain $i$.

\textbf{IWTGP. } It is not obvious how to factor out computations that does not depend on the test data in the case of IWTGP, since the computational extensive factor(i.e., $({\textbf{W}^i}^\frac{1}{2} $ $\textbf{K}_X^i {\textbf{W}^i}^\frac{1}{2} + \lambda_x^i \textbf{I})^{-1}$,  $({\textbf{W}^i}^\frac{1}{2} \textbf{K}_Y^i {\textbf{W}^i}^\frac{1}{2} + \lambda_y^i \textbf{I})^{-1}$ ) does depend on the test set since $\textbf{W}^i$ is computed on test time.  To help factor out the computation, we used linear algebra to show that
\begin{equation}
(\textbf{D A D} + \lambda \textbf{I})^{-1} =   \textbf{D}^{-1} \textbf{A}^{-1} \textbf{D}^{-1} - \frac{\lambda \textbf{D}^{-2} \textbf{A}^{-2} \textbf{D}^{-2}} {1 + \lambda \cdot  tr(\textbf{D}^{-1} \textbf{A}^{-1} \textbf{D}^{-1})}
\label{eq:lem}
\end{equation}
where $\textbf{D}$ is a diagonal matrix,  $I$ is the identity matrix, and $tr(\textbf{B})$ is the trace of matrix $\textbf{B}$.  \begin{proof}
Kenneth Miller~\cite{Miller1981} proposed the following Lemma on Matrix Inverse. 
 \begin{equation}
 (G+H)^{-1} = G^{-1} -\frac{1}{1+tr(G H^{-1})} G^{-1} H G^{-1} 
 \end{equation}
  Applying Miller's lemma, where  $G = D A D$ and $H = \lambda I$, leads directly to Eq.~\ref{eq:lem}. %$(\textbf{D A D} + \lambda \textbf{I})^{-1} =   \textbf{D}^{-1} \textbf{A}^{-1} \textbf{D}^{-1} - $ $\frac{\lambda (\textbf{D}^{-1} \textbf{A}^{-1} \textbf{D}^{-1})^{2}} {1 + \lambda^{-1} \cdot  tr(\textbf{D}^{-1} \textbf{A}^{-1} \textbf{D}^{-1})}$.
 \end{proof}

Mapping $\textbf{D}$ to ${\textbf{W}^i}^\frac{1}{2}$ \footnote{$\textbf{W}$ is a diagonal matrix}, $\textbf{A}$ to either of $\textbf{K}_X^i$ or $\textbf{K}_Y^i$, we can compute $\mathcal{M}^i = \{{\textbf{K}_X^i}^{-1}, {\textbf{K}_Y^i}^{-1}\}$. Having computed $\textbf{W}^i$ on test time,  $({\textbf{W}^i}^\frac{1}{2} \textbf{K}_X^i {\textbf{W}^i}^\frac{1}{2} + \lambda_x \textbf{I})^{-1}$,  $({{\textbf{W}^i}}^\frac{1}{2} \textbf{K}_X {\textbf{W}^i}^\frac{1}{2} + \lambda_x \textbf{I})^{-1}$ could be computed in quadratic time given  $\mathcal{M}^i$  following equation ~\ref{eq:lem}, since the inverse and the power of ${\textbf{W}^i}^\frac{1}{2}$ has linear computational complexity since it is diagonal.

\ignore{ information are created and stored in DDTGP model as illustrated in Algorithm ~\ref{alg:ddtgpm}.
\begin{algorithm}
{\textbf{Input:} Clusters ${\{C_k\}}_{k=1}^{NC} $}
\KwOut{Overlapping subdomains ${\{D_k\}}_{k=1}^{NC}$}
\ForEach{Cluster $C_k$}{
Compute the closest $OCC$ clusters ${\{{{OVC_K}_i}\}}_{i=1}^{OCC}$ based on $DK_i = \| \boldsymbol{\mu}_k- \boldsymbol{\mu}_i \|$ , $i\neq k$\\
Let $LK_i = 1/DK_i,  {WK_i} =  \frac{LK_i}{\sum_{l=1}^{OCC} LK_l}$  ${i=1 : OCC}$\\
Let ${NPK_i} =  floor(WK_i * OPC)$, ${i=1 : OCC}$ \\
Let $ExKPts = OPC - \sum_{l=1}^{OCC} NPK_l$ \\
Let ${NPK_i}$ =  $NPK_i +1$ , $ i=1 : ExKPts $\\
$D_k =  C_k$ \\
\For{i=1 : OCC} { $Ps_i = KNN({OVC_K}_j,NPK_i )$ \\ $D_k = D_k \cup Ps_i$ }
}
\Comment{where KNN is the K-nearest neighbors algorithms. For high performance calculation of $KNN$, we use FLANN \cite{flann09} to calculate $KNN$.}
\caption{Subdomains Generation (Note: All ${\{D_k\}}_{k=1}^{NC}$ are stored as indices to $X$).  }
\label{alg:sd}
\end{algorithm}}

\begin{comment}
\begin{algorithm}
\KwIn{Overlapping subdomains ${\{D_k\}}_{k=1}^{NC}$}
\KwOut{DDModel (Domain Decomposition TGP Model) }
\For{i=1: NC}{ $M_i.SubDomain = D_k $  \\
$M_i.K_X =  (K_X)_K $ \\
\% where $(K_X)_K$ is the input kenel matrix evaluated on the points of  $D_k$ only. \\
$M_j.K_Y =  (K_Y)_K $\\
\% where $(K_Y)_K$ is the output kenel matrix evaluated on the points of  $D_k$ only. \\
$M_i.{{K_X}_{inv}} =( K_X + \lambda_x I)^{-1}$ \\
$M_i.{{K_Y}_{inv}} =( K_Y + \lambda_y I)^{-1}$ \\
$M_i.{{\boldsymbol{\mu}_D}_X} =\boldsymbol{\mu}_{D_i}$ \Comment{$d_x \times 1 $ vector}\\ 
$M_i.{{\Sigma_D}_X}^{-1} = {{{\Sigma}_D}_i}^{-1} $  \Comment{$d_x \times d_x $ matrix} \\
\Comment{where $\boldsymbol{\mu}_{D_i}, {{{\Sigma}_D}_i}$ are the MLE estimate mean and covariance estimate for input points in $D_i$ }\\
$DDModel.Models[i] = M_i$;
}
\caption{DDTGP Model Generation}
\label{alg:ddtgpm}
\end{algorithm}
\end{comment}
\ignore{
\begin{algorithm}
\KwIn{Overlapping subdomains ${\{D_k\}}_{k=1}^{NC}$}
\KwOut{DDModel (Domain Decomposition TGP Model) }
\For{i=1: NC}{ $M_i.SubDomain = D_k $  \\
$M_i.K_X =  (K_X)_k $ \\
\% where $(K_X)_k$ is the input kenel matrix evaluated on the points of  $D_k$ only. \\
$M_j.K_Y =  (K_Y)_k $\\
\% where $(K_Y)_k$ is the output kenel matrix evaluated on the points of  $D_k$ only. \\
$M_i.{{K_X}_{inv}} =( M_i.K_X + \lambda_x I)^{-1}$ \\
$M_i.{{K_Y}_{inv}} =( M_i.K_Y + \lambda_y I)^{-1}$ \\
$M_i.{{\boldsymbol{\mu}_D}_X} =\boldsymbol{\mu}_{D_i}$ \Comment{$d_x \times 1 $ vector}\\ 
$M_i.{{\Sigma_D}_X}^{-1} = {{{\Sigma}_D}_i}^{-1} $  \Comment{$d_x \times d_x $ matrix} \\
\Comment{where $\boldsymbol{\mu}_{D_i}, {{{\Sigma}_D}_i}$ are the MLE estimate mean and covariance estimate for input points in $D_i$ (used in Mode 3) }\\
$DDModel.Models[i] = M_i$;
}
\caption{DDTGP Model Generation}
\label{alg:ddtgpm}
\end{algorithm}
}
\ignore{
\begin{algorithm}
\KwIn{query point X, DDTGPModel}
\KwOut{Closest subdomains }
Let $KNN_x =  KNN(x, X)$  \Comment{X is all input training points.} \\
Let ${L_{KNN_x}} =  Clusters(KNN_x)$ \Comment{where Clusters($KNN_x$) returns the corresponding clusters for the returned points.} \\ 
$FT = frequencyTable(L_{KNN_x})$ \Comment{where $frequencyTable(L_{KNN_x})$ returns an frequency of each clusters in $L_{KNN_x}$. }\\
$[sortedFT, sortedInd]  = sort(FT$) \\
$closestDomains = sortedInd[1:CDC] $\\ 
$closestFTS = sortedFT[1:CDC]$ \\
\caption{Mode 2 closest clusters algorithm}
\label{alg:mode2FC}
\end{algorithm}
}
\begin{comment}

\begin{figure}[h!]
  \null\hfill \algTwo{alg2} \hfill \algThree{alg3} \hfill\null\par 
\end{figure}
\end{comment}

\subsection{Prediction}
\label{sec:pred}

ODC-Prediction \ignore{under our  framework }is performed in  three steps. 
%\subsection{Finding the closest subdomains}
%\label{ss:fcc}

\noindent\textbf{{ (1) Finding the closest subdomains.}}
The closest  $K' \ll K$  subdomains  are determined based on the covariance norm of the displacement of  the test input from the means of the subdomain distribution (i.e. $\|\mathbf{x}-\boldsymbol{\mu
}'_{i} \|_{{\Sigma'_{i}}^{-1}}, i = 1: K$, where \ignore{$\boldsymbol{\mu
}'_{i}, {\Sigma'_{i}}^{-1}$ are as described in section ~\ref{sec:train2}),} $\|\mathbf{x}-\boldsymbol{\mu
}'_{i} \|_{{\Sigma'_{i}}^{-1}}$  = $ (\mathbf{x}-\boldsymbol{\mu
}'_{i})^T {\Sigma'_{i}}^{-1}  (\mathbf{x}-\boldsymbol{\mu
}'_{i})$. The reason behind using the covariance norm is that it captures details of the density of the distribution in all dimensions. Hence, it better models $p(\mathbf{x}|D_i)$, indicating better prediction of $\mathbf{x}$ on $D_i$.\ignore{, which indicates better prediction of corresponding $\mathbf{x}$ on $D_i$}
\ignore{
\begin{itemize}
\itemsep0em 
%\subsubsection{Mode 1} 
\item \textbf{Mode 1: } Closest clusters are determined based on the distance to the means of the clusters (i.e. $\|x-\mu_k \|, k = 1: NC$ ).
%\subsubsection{Mode 2}
\item \textbf{Mode 2: } This mode has an extra prediction parameter $M2K (Mode2 K)$. Algorithm ~\ref{alg:mode2FC}  details how closest clusters are selected in Mode 2.
%\subsubsection{Mode 3}
\item \textbf{Mode 3: } Closest clusters are determined based on the covariance norm of the displacement of  the test input from the means of the clusters (i.e. $\|x-\mu_{D_k} \|_{{\Sigma_{D_k}}^{-1}}, k = 1: NC$ . where $\mu_{D_k}, {\Sigma_{D_k}}^{-1}$ are as described in Algorithm ~\ref{alg:ddtgpm}), $\|x-\mu_{D_k} \|_{{\Sigma_{D_k}}^{-1}}$  = $ (x-\mu_{D_k})^T {\Sigma_{D_k}}^{-1}  (x-\mu_{D_k})$ 
\end{itemize}
}
\ignore{
\begin{algorithm}
\KwIn{query point X, DDTGPModel}
\KwOut{Closest subdomains }
Let $KNN_x =  KNN(x, X)$  \Comment{X is all input training points} \\
Let ${L_{KNN_x}} =  Clusters(KNN_x)$ \Comment{where Clusters($KNN_x$) returns the corresponding clusters for the returned points} \\ 
$FT = frequencyTable(L_{KNN_x})$ \Comment{where $frequencyTable(L_{KNN_x})$ returns an frequency of each clusters in $L_{KNN_x}$ }\\
$[sortedFT, sortedInd]  = sort(FT$) \\
$closestDomains = sortedInd[1:CDC] $\\ 
$closestFTS = sortedFT[1:CDC]$ \\
\caption{Mode 2 closest clusters algorithm}
\label{alg:mode2FC}
\end{algorithm}
}
%\subsection{Closest subdomains Prediction}
%\label{ss:csdp}

\noindent\textbf{{(2) Closest subdomains Prediction.}}
Having determined the closest subdomains, predictions are made for each of the closest clusters. We denote these predictions as ${\{Y^i_{x_*}\}}_{i=1}^{K'}$.  Each of these prediction are computed \ignore{by the following equation,} according to the selected kernel machine. For GPR, predictive mean and  variance are  $O(M\cdot d_X)\,$  and  $O(M^2 \cdot d_Y)\,$ respectively, for each output  dimension.  For TGP,   the prediction is  $O(l_2 \cdot M^2 \cdot d_Y)$; see Eq ~\ref{eq:tgp}. \ignore{, where $l_2$ is the number of iterations for  convergence, since TGP prediction has no closed form expression Similarly, any arbitrary kernel machine saves all factor of the prediction function that only depends on the training data.} \ignore{????? Add Example Kernel machines here????}

\noindent \textbf{{(3) Subdomains weighting and Final prediction.}}
The final predictions are formulated as $\textbf{Y}(x_*) = \sum_{i=1}^{K'} a_i \textbf{Y}^i_{x_*}, a_i>0, \sum_{i=1}^{K'} {a_i} =1 $. ${\{a_i\}}_{i=1}^{K'}$ are computed as follows.  Let the distribution of domain ${\{D^i_{x_*} = \|x-\mu'_{i} \|_{{\Sigma'_{k}}^{-1}}\}}_{i=1}^{K'}$ denotes to the distances to the closest subdomains, ${\{L^i_{x_*} = {1}/{D^i_{x_*}}\}}_{i=1}^{K'}$, $a_i ={L^i_{x_*}}/{\sum_{i=1}^{K'} {L^i_{x_*}}} $. \ignore{; see table ~\ref{tab:thcomp} for detailed complexity of the ODC framework in GPR and TGP. }

It is not hard to see that when $K' = 1$, the prediction step reduces to regression using the closest subdomain to the test point. However it is  reasonable in most of the prior work to make prediction using the closest model, we generalized it to $K'$ closest kernel machines and combining their predictions,  so as to study how consistency of the combined prediction behaves as the overlap increases (i.e., $p$); see the experiments.\ignore{ Table ~\ref{tab:thcomp} illustrates the computational complexity of the framework. }


\ignore{\begin{itemize}
\itemsep0em 
%\subsubsection{Mode 1} 
\item \textbf{Mode 1: } Let ${\{D_{x_i}\}}_{i=1}^{CDC}$ denotes to the distances to the closest clusters, ${\{L_{x_i} = \frac{1}{D_{x_i}}\}}_{i=1}^{CDC}$. Then $W_i =\frac{L_x(i)}{\sum_{i=1}^{CDC} {L_x(i)}} $
\item \textbf{Mode 2: }$W_i = \frac{closestFTS(i)}{\sum_{i=1}^{CDC} closestFTS(i)}$
 where closestFTS is as computed in subsection ~\ref{ss:fcc}.
\item \textbf{Mode 3: } Let ${\{D_{x_i} = \|x-\mu_{D_k} \|_{{\Sigma_{D_k}}^{-1}}\}}_{i=1}^{CDC}$ denotes to the distances to the closest clusters, ${\{L_{x_i} = \frac{1}{D_{x_i}}\}}_{i=1}^{CDC}$. Then  $W_i =\frac{L_x(i)}{\sum_{i=1}^{CDC} {L_x(i)}} $.
\end{itemize}}
\ignore{
noindent \textbf{Computational Complexity:} Overall computational complexity of the prediction = $O( n_{Tst}* CDC * (\frac{N_{Tr}}{NC}) ^2)$ (i.e. quadratic complexity) which is significantly better than cubic complexity of the original TGP model. where $n_{Tst}$ is the number of the test points, $N_{Tr}$ is the number of points in the training data.
\subsection{Data Bias handling ($CovShift$ Flag = $True$)}
\input{databias}}

\section{Experimental Results}
\label{sec:6}

%\begin{figure*}[ht!]
%\centering
%\vspace{-3mm}
% \begin{subfigure}[b]{1.0\textwidth}
% \hspace{-10mm} \includegraphics[width=1.0\textwidth,height=0.32\textwidth]{figGPR800HEva2.eps}
% \vspace{-4mm}
%   \hspace{-10mm}\caption{GPR-ODC (M=800) }
%\end{subfigure}
%\begin{subfigure}[b]{1.0\textwidth} 
% \hspace{-10mm}\includegraphics[width=1.0\textwidth,height=0.32\textwidth]{figTGP800HEva2.eps}
%  %\hspace{-15mm}
%  \vspace{-4mm}
%\caption{TGP-ODC (M=800)}
%\end{subfigure}
%  \vspace{-5mm}
%\caption{ODC framework Parameter Analysis of GPR and TGP  on Human Eva Dataset (best seen in color)}
%\label{fig:ODCAnalysis}
%\vspace{-5mm}
%\end{figure*}

\ignore{
In this section, we present our experimental evaluation. We start by evaluating the performance of the clustering algorithms, we proposed for the purpose of domain decomposition. Then, we discuss the performance of our ODC method on TGP and IWTGP only for space limit purpose. However, it is also applicable to GPR as presented. More results are attached in the supplementary materials}
\ignore{In this section, we present the experimental validation of our ODC-framework. We firstly present the datasets then the details of our experiments. }
%\subsection{Datasets, Features and Error Measures}

\textbf{Equal-Size Kmeans Step Experiment:}  We also tried another variant for Ekmeans that we call Iterative Minimum-Distance Assignments EKmeans (IMDA- Ekmeans).  Note that the algorithm presented earlier in the paper is denoted as Assign and Balance Kmeans (AB-Kmeans).  The IMDA- Ekmeans  algorithm works as follows. We initialize a pool of unassigned points $\tilde{X}  =  X$ and initialize all clusters as empty.  Given the means computed from the previous update steps, we compute the distances $d(\mathbf{x}_i,\mu_j)$ for all points/center pairs. We iteratively pick the minimum distance pair 
\[\small (\mathbf{x}_p,\mu_l)  : d(\mathbf{x}_p,\mu_l) \le d(\mathbf{x}_i,\mu_j)  \forall x_i \in \tilde{X} \text{and}  |C_l| < N/K \]
and assign point $x_p$ to cluster $l$. The point is then removed from the pool of unassigned points. if  $|C_l| = N/K$,  then it is marked as balanced and no longer considered. The process is repeated until the pool is empty; see Algorithm~\ref{alg:ddclusterALg2}. 

\begin{algorithm}[b!]
\KwIn{$\textbf{X} (N \times d_x ),{\{\boldsymbol{\mu}_i\}}_{i=1}^{K}$}
\KwOut{labels}
1- Create a matrix $\textbf{D} \in R^{N \times K}$, where $\textbf{D}[i,j]$ is the distance between the $i^{th}$ point to the $j^{th}$ cluster center.\\
2- Get the coordinate $(i_*,j_*)$ that maps the smallest distance in $\textbf{D}$.\\
3- Remove the $i_*^{th}$ row from matrix $\textbf{D}$ and mark it as assigned to the $j^{th}$ cluster\\
4- If the size of the cluster $j$ achieves the ideal size (i.e. ~ $n/K$), then remove the $j^{th}$ column from matrix $\textbf{D}$.\\
5- Go to step 2 if there is still unassigned points
\caption{Iterative Minimum-Distance Assignments (IMDA) k-means: Assignment Step}
\label{alg:ddclusterALg2}
\end{algorithm}

%\subsection{Iterative Minimum-Distance Assignments (IMDA) k-means EKmeans}

%\subsection{Comparison between IMDA and Assign and Balance (AB) EKmeans(presented in the paper) }
Table ~\ref{tab:clustering} presents the average cost over 10 runs of IMDA-Ekmeans and AB-Ekmeans algorithms. We initialize both the AB-Ekmeans and  IMDA-EKmeans algorithms by the cluster centers computed by running the standard k-means. As illustrated in table~\ref{tab:clustering}, the AB-Ekmeans  outperforms IMDA-Ekmeans in these experiments, which motivated us to utilize AB Ekmeans, which is presented in the paper, against  IMDA-Ekmeans under our ODC prediction framework.  Our interpretation for these results is because AB-Ekmeans initializes the assignment with an assignment that minimizes  the cost $J(C) = \min \sum_{j=1}^K \sum_{\mathbf{x_i}\in C_j} d(\mathbf{x_i}, \mu_j)$ given the cluster centers and then balance the clusters.  In all the following experiments, we uses AB-EKmeans due to its clear superior performance to IMDA-EKmeans.

\begin{table}[htbp!]
  \centering
  \caption{$J(C)$ of AB-kmeans and  IMDA-kmeans on a dataset of 10,000 random 2D points, averaged over 10 runs}
    \label{tab:clustering}%
    \begin{tabular}{|c|c|c|c|}
    \hline
          & \textbf{K=5} & \textbf{K = 10} & \textbf{K=50} \\
              \hline
    \textbf{AB-kmeans} & 1077.3 & 540.241 & 105.505 \\
    \textbf{IMDA-kmeans} & 1290.6 & 657.446 & 122.006 \\
    \textbf{Error Reduction} & 16.53\% & 17.83\% & 13.52\% \\
    \bottomrule
    \end{tabular}% 
\end{table}% 

\begin{figure}[th!]
\includegraphics[width=0.5\textwidth]{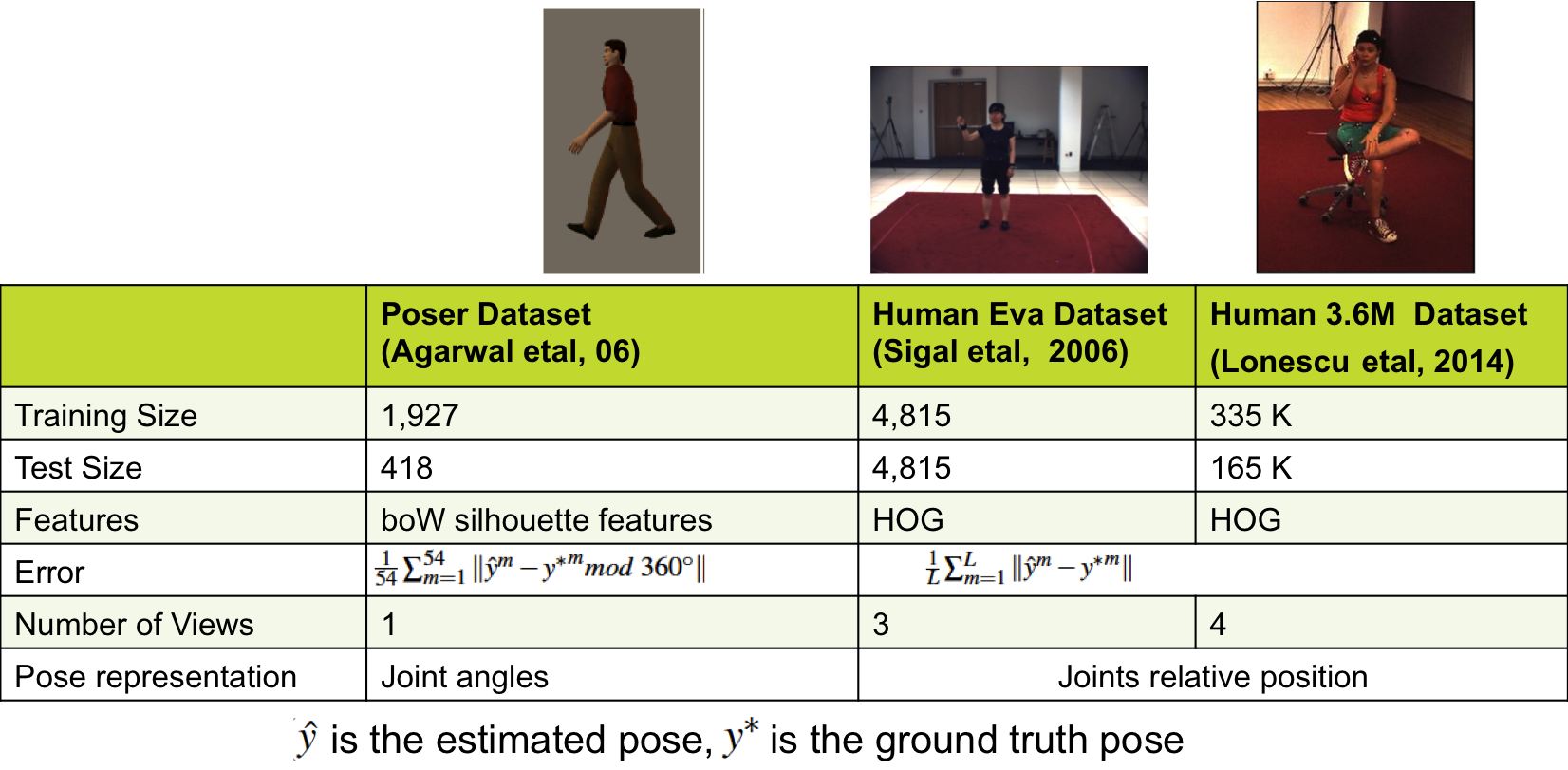}
\caption{Datasets, Representations, and  Features}
\label{fig_datasets}
\end{figure}

\noindent \textbf{Datasets and Setup. } We evaluated our framework on three human pose estimation datasets, Poser, HumanEva, and Human3.6M; see Fig.~\ref{fig_datasets} for summary of setup and representation for each. \ignore{We here briefly describe the configuration, features, pose representation, and error measure for each dataset.} Poser dataset~\cite{Trig06} consists of 1927 training and 418 test images\ignore{, which are synthetically generated and tuned to unimodal predictions}. The image features, corresponding to bag-of-words representation with silhouette-based shape-context features. The error is measured by the root mean-square error (in degrees), averaged over all joints angles, and is given by: \small$Error(\hat{y}, y^*) = \frac{1}{54} \sum_{m=1}^{54} \| {\hat{y}}^m  - {y^*}^m mod$  $360^\circ \|$\normalsize , where $\hat{y} \in R^{54}$ is an estimated pose vector, and $y^* \in R^{54}$ is a true pose vector. HumanEva datset \cite{SigalBB10} contains synchronized multi-view video and Mocap data of 3 subjects performing multiple activities. We use  \ignore{histogram of oriented gradient (HOG)} HOG features~\cite{HOG05} (\small$\in R^{270}$\normalsize) proposed in \cite{Bo:2010}. We use training and validations subsets of HumanEva-I and only utilize data from 3 color cameras with a total of 9630 image-pose frames for each camera. This is consistent with experiments in \cite{Bo:2010} and \cite{Yamada:2012}. We use half of the data for training and half for testing. Human3.6M~\cite{human3m} is a dataset of millions of Human poses. We managed to evaluate our proposed ODC-framework on six Subjects (S1, S2, S6, S7, S8, S9) from it,  which is $\approx$ 0.5 million poses. We split them into  67\% training 33\% is testing. HOG features are extracted for 4 image-views for each  pose and concatenated into 3060-dim vector. Error for each pose, in both HEva (in $mm$) and Human 3.6 (in $cm$),  is measured as \ignore{average Euclidean distance:}\small$Error(\hat{y},y^*) = \frac{1}{L} \sum_{m=1}^{L} \|\hat{y}^m - {y^*}^m\|$\normalsize\ignore{, where $\hat{y}$ is an estimated pose vector, and $y^*$ is a true pose vector}.

\begin{figure*}[t!]
\centering
\begin{tabular}{c}
{{{ \includegraphics[width=1.0\textwidth,height=0.34\textwidth]{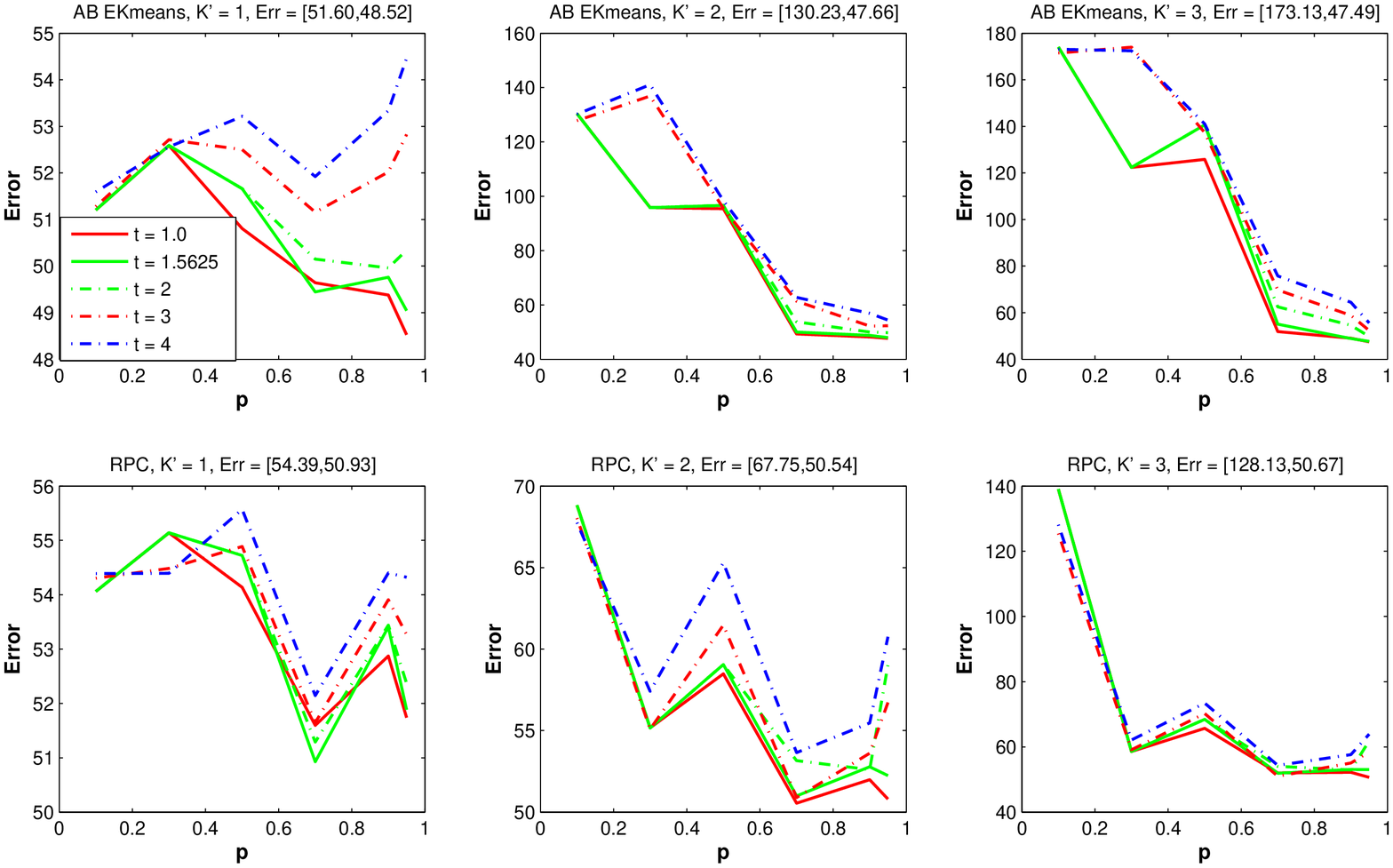}}}}\\
 {(a) GPR-ODC (M=800) } \\
{{  \includegraphics[width=1.0\textwidth,height=0.34\textwidth]{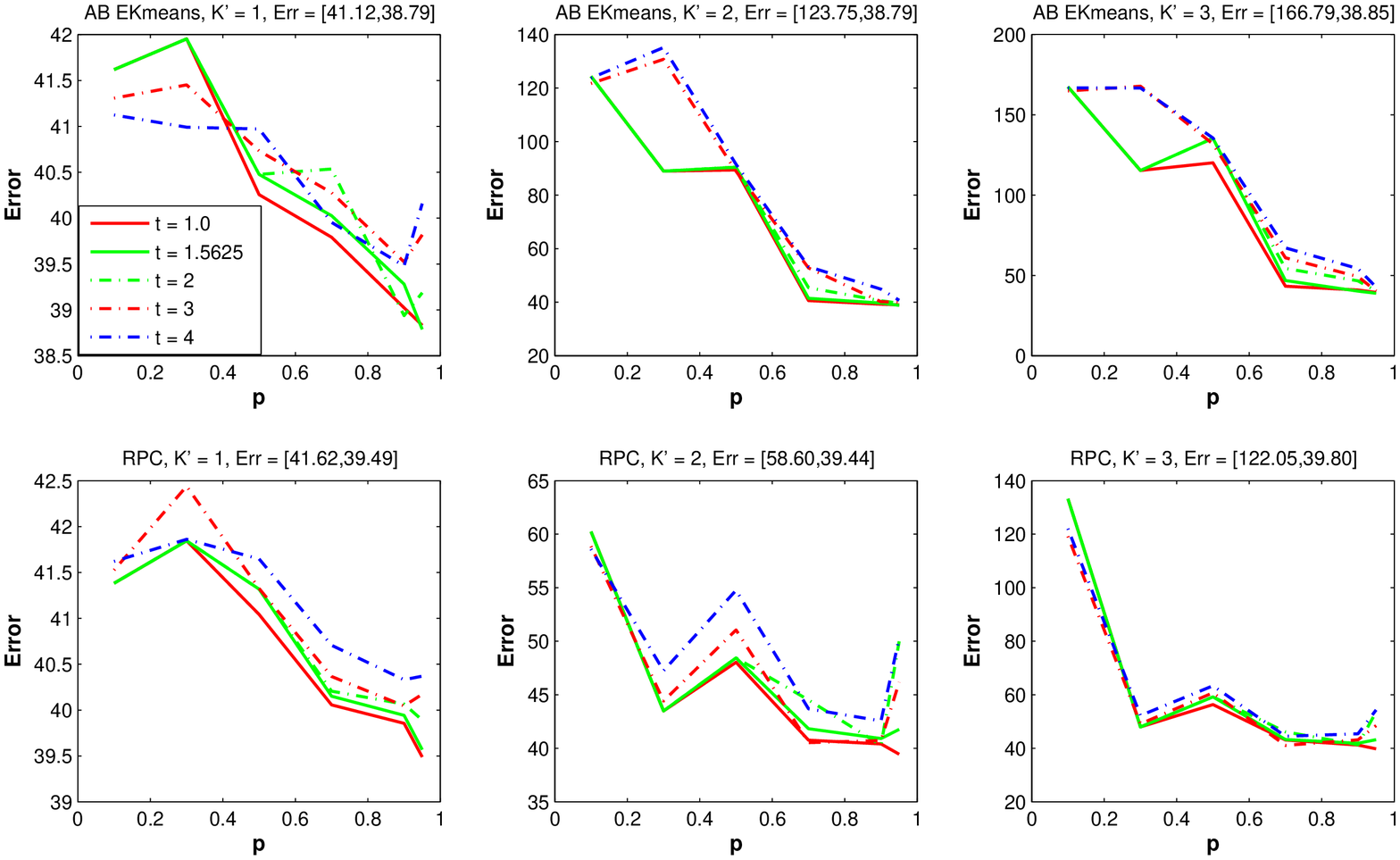}}} \\
{ (b) TGP-ODC (M=800)} \\
\end{tabular}
\vspace{2mm}
\caption{ODC framework Parameter Analysis of GPR and TGP  on Human Eva Dataset}
\label{fig:ODCAnalysis}
%\label{fig:teaser}
\end{figure*}

\begin{table*}[t!]
  \centering
      %\vspace{4pt}
  \caption{Error \& Time on Poser and Human Eva datasets (Intel core-i7 2.6GHZ), M = 800}
    \scalebox{0.8}{
    \begin{tabular}{|l|l|lll|lll|}%ccc}
    \hline
          &   & \textbf{Poser} & \textbf{} &       & \textbf{HumanEva} &       &      \\% & \textbf{Human 3.6} &       &  \\
    \hline
         & & \textbf{Error (deg)} & \textbf{Training Time} & \textbf{Prediction Time} & \textbf{Error (mm) } & \textbf{Training Time} & \textbf{Prediction Time} \\%& \textbf{Error} & \textbf{Training Time} & \textbf{Prediction Time} \\
 \textbf{TGP}   & \textbf{NN ~\cite{Bo:2010}} & 5.43       &  -     &    188.99 sec   & \textbf{38.1}      &  -     & 6364 sec \\%       &       &       &  \\
%   & \textbf{Full} & \textbf{5.32 }     &    -  &   \textbf{27.5}  sec  &    \textbf{40.3 }  &  -   &  \textbf{1010} sec\\%       &       &       & 
  & \textbf{ODC ($p= 0.9, t=1, K'=1$)-Ekmeans} & \textbf{5.40 }     &      (3.7 +25.1 ) sec  &   \textbf{16.5}  sec  &    \textbf{38.9 }  &  (2001 + 45.4) sec    &  \textbf{298} sec\\%       &       &       &  \\
    & \textbf{ODC ($p= 0, t=1, K'=1$)-Ekmeans} &   7.60    &    (3.9 + 1.33) sec   &   14.8 sec  &    41.87   & (240 + 4.9 ) sec       &  257 sec \\%      &       &       &  \\
      &  \textbf{ODC ($p= 0.9, t=1, K'=1$)-RPC} & 5.60      &      (0.23 +41.6 ) sec  &   15.8 sec  &   39.9  &    ( 0.45 + 49.1) sec     & 277 sec\\%       &       &       &  \\
  &  \textbf{ODC ($p= 0, t=1, K'=1$)-RPC} &   7.70    &   (0.15 + 1.7) sec   &   13.89 sec  &  42.32    &  (0.19 + 5.2)      sec &  242 sec\\%      &       &       &  \\
      \hline
  \textbf{GPR}  & \textbf{NN} & 6.77      &   -    &  24 sec     &   54.8    &      - &    618  sec \\%&       &       &  \\
%   & \textbf{Full}  &  {6.10}      & -   &       \textbf{0.51}  sec & {59.62}  &   -  & \textbf{10.2} sec \\% 
   & \textbf{ODC ($p= 0.9, t=1 , K'=1$)-Ekmeans} &  \textbf{6.27}      &  (3.7 +11.1 ) sec  &       \textbf{0.56}  sec & \textbf{49.3}  &   (2001 + 42.85)sec & \textbf{79} sec \\%       &       &       &  \\
   & \textbf{ODC($p= 0.0, t=1 , K'=1$)-Ekmeans} & 7.54      &   ( 3.9 + 1.38 sec) &    0.35 sec   & 49.6  &  (240 + 6.4) sec  &  48 sec\\%      &       &       &  \\
  & \textbf{ODC ($p= 0.9, t=1 , K'=1$)-RPC} &  6.45      &  (0.23 +17.3 ) sec  &       0.52  sec & 52.8  & (0.49 + 46.06) sec     &  64 sec\\%       &       &       &  \\
    & \textbf{ODC ($p= 0.0, t=1 , K'=1$)-RPC  = ~\cite{Chalupka:2013}} &   7.46    &   (0.15 + 1.5) sec &    0.27 sec   & 54.6  &  (0.26 + 4.6 ) sec & 44 sec\\%      &       &       &  \\
    & \textbf{FIC ~\cite{fic06}} &   7.63   &   (- + 20.63)   &    0.3106     &   68.36   &  -     & 102  sec\\%      &       &       &  \\
    \hline
    \end{tabular}}%
  \label{tab:tblRes}%
  \vspace{-5mm}
\end{table*}%

%\subsection{Experiments}
%\def\arraystretch{0.8}
There are four control parameters in our ODC framework: $M$, $p$, $t$, and $K'$.\ignore{We started by performing inherent analysis of these parameters.} Figure~\ref{fig:ODCAnalysis}  shows our parameter analysis with different values of $p$, $t$ and $K'$ on HumanEva dataset  for GPR and TGP as local regression machines, where $M=800$.  Each sub-figure consists of six plots in two rows. The first row indicates the results using AB-Ekmeans clustering scheme, while the second row shows the results for RPC clustering scheme. Each row has three plots, one for $K'=1$, $2$, and $3$ respectively. 
\ignore{\begin{figure}[t!]
\centering
\begin{tabular}{c}
\bmvaHangBox{{{ \includegraphics[width=1.0\textwidth,height=0.34\textwidth]{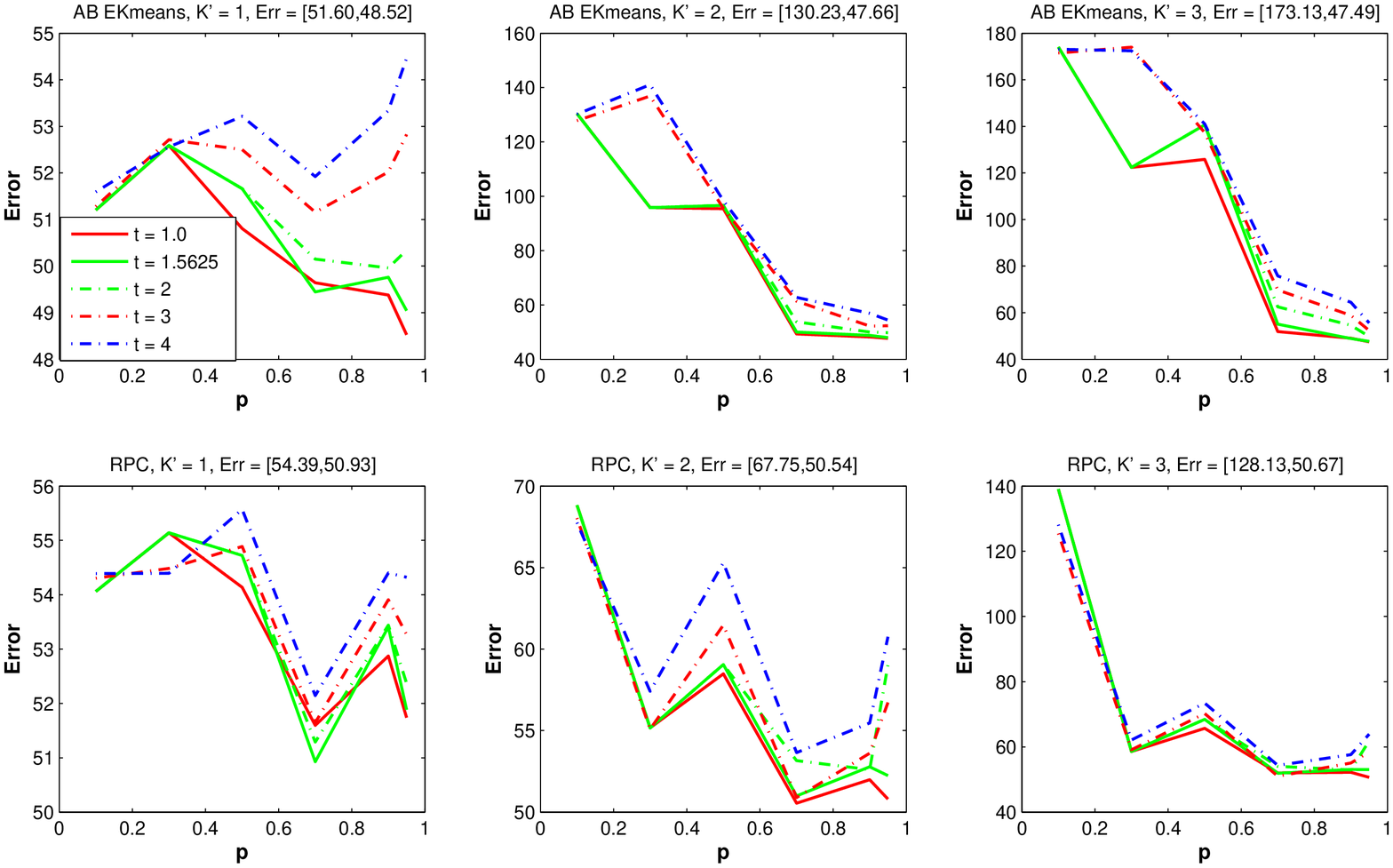}}}} \\
\vspace{-7mm}
 {\tiny(a) GPR-ODC (M=800) } \\
\bmvaHangBox{{  \includegraphics[width=0.\textwidth,height=0.34\textwidth]{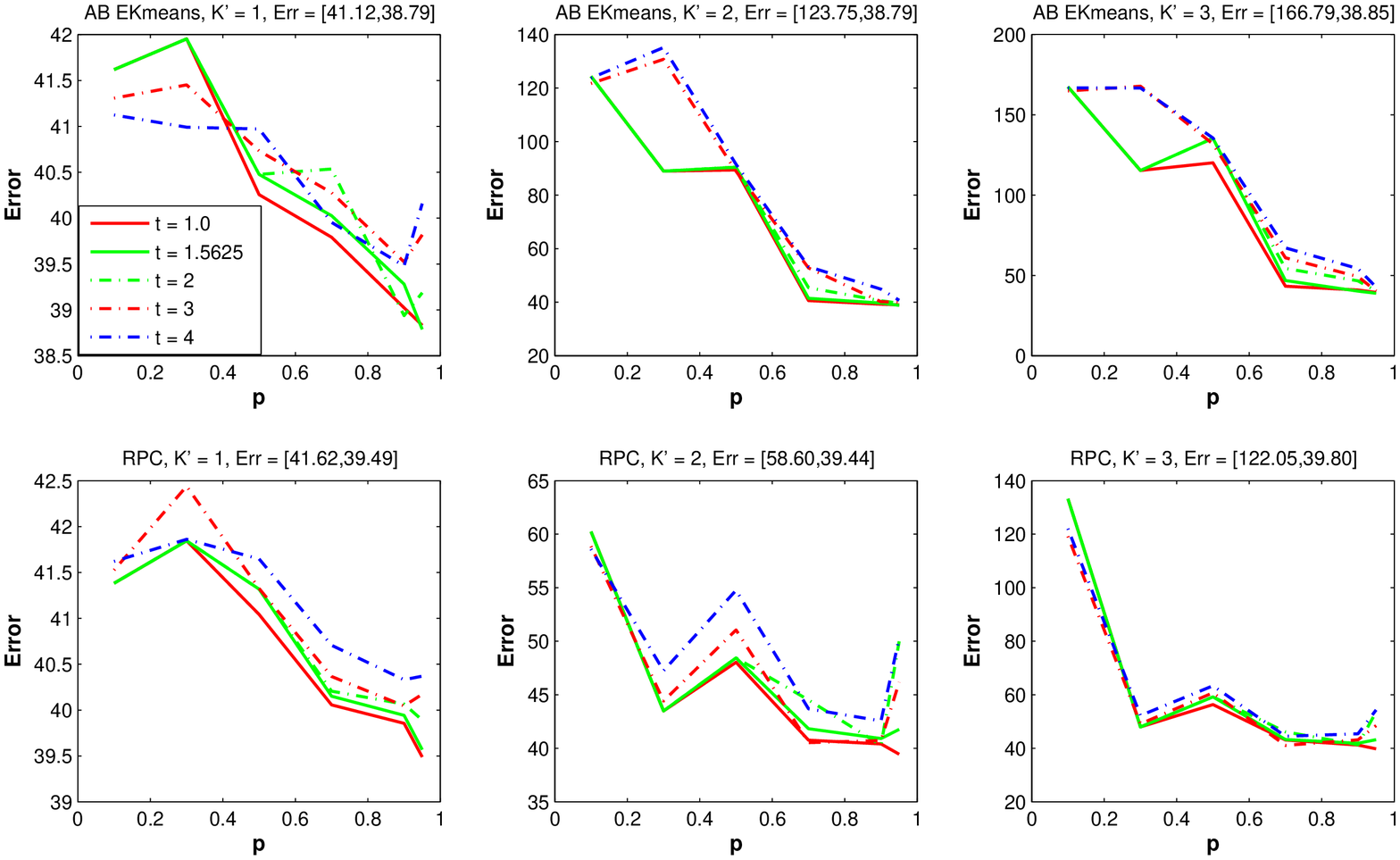}}}\\
 \vspace{-7mm}
{ \tiny (b) TGP-ODC (M=800)} \\
\ignore{\bmvaHangBox{{  \includegraphics[width=1.0\textwidth,height=0.34\textwidth]{figTGP400HEva.eps}}}\\
(c) TGP-ODC (M=400) \\}
\end{tabular}
\caption{ODC framework Parameter Analysis of GPR and TGP  on Human Eva Dataset (seen in color)}
\label{fig:ODCAnalysis}
\vspace{-15mm}
\end{figure}}
\begin{comment}
\begin{figure}[h!]
\centering
\begin{tabular}{cc}
 {\footnotesize (a) GPR-ODC (M=800) } &
\bmvaHangBox{{{     \includegraphics[width=1.0\textwidth,height=0.34\textwidth]{figGPR800HEva.eps}}}} \\
 { \footnotesize (b) TGP-ODC (M=800)}  &
\bmvaHangBox{{  \includegraphics[width=1.0\textwidth,height=0.34\textwidth]{figTGP800HEva.eps}}}\\
\ignore{\bmvaHangBox{{  \includegraphics[width=1.0\textwidth,height=0.34\textwidth]{figTGP400HEva.eps}}}\\
(c) TGP-ODC (M=400) \\}
\end{tabular}
\vspace{2mm}
\caption{Overlapping Domain Cover Parameter Analysis of GPR and TGP  on Human Eva Dataset (best seen in color)}
\label{fig:ODCAnalysis}
\vspace{-15mm}
%\label{fig:teaser}
\end{figure}
\end{comment}
Each plot shows the error of different $t$ against $p$ from 0 to 0.95; i.e., it shows how the overlap affects the performance for different values of $t$. Each plot shows, on its top caption, the minimum and the maximum overlap regression errors where $t \to 1$. Looking at these plots, there are a number of observations:

\noindent\textbf{(1)} As $t \to 1$ (the solid red line), the behavior of the error tends to reduce as $p$ increases, i.e., the overlap.

\noindent\textbf{(2)} Comparing different $K'$, the behavior of the error indicates that combining multiple predictions (i.e., $K'=2$ and $K'=3$), gives poor performance, compared with  $K'=1$,  when the overlap is small. However, all of them, $K'=1$, $2$, and  $3$, performs well as  $p \to 1$; see column 2 and 3 in Fig.~\ref{fig:ODCAnalysis} and Fig.~\ref{fig_K_all}.  This indicates consistent prediction of neighboring subdomains as $p$ increases; see also Fig.~\ref{fig_K_each} for side by side comparison of different $K'$. The main reason behind this behavior is  that as $p$ increases, the local models of the  neighboring subdomains normally share more training points on their boundaries, which is reflected as shared constraints during the training of these models making them more consistent on prediction. 

\noindent\textbf{(3)} Comparing the first row to the second row in each subfigure, it is not hard to see that our AB-Ekmeans partitioning scheme consistently \ignore{ends up with better error rates compared with} outperforms RPC \cite{Chalupka:2013}, \eg the error in cases of  GPR (M=800) is  47.48mm for AB-EKmeans and 50.66mm for RPC, TGP (M=800) is  38.8mm for AB-EKmeans and 39.8mm for RPC. This problem is even more severe when using smaller $M$, \eg the error in case of \ignore{ GPR (M=400) is 45.20mm for EKmeans and 50.56mm for RPC, } TGP (M=400) is 39.5mm for EKmeans and 47.5mm for RPC; see a detailed plot for M=400 in Fig.~\ref{fig:ODCAnalysis400}.  We noticed sigficant drop in the performance as M decreases. For instance when $M=200$, The error for TGP best performance increased to 43.88mm instead of 38mm for $M=800$.  

\noindent\textbf{(4)} TGP gives better prediction than GPR (i.e., 38mm using TGP compared with 47mm using GPR). 

\noindent\textbf{(5)} As $M$ increases, the prediction error decreases. For instance, when $M=200$, The error for TGP best performance increased to 43.88mm instead of 38.9mm for $M=800$. We found these observation to be also consistent on Poser dataset. 
\begin{figure*}[t!]
\includegraphics[width=1.0\textwidth]{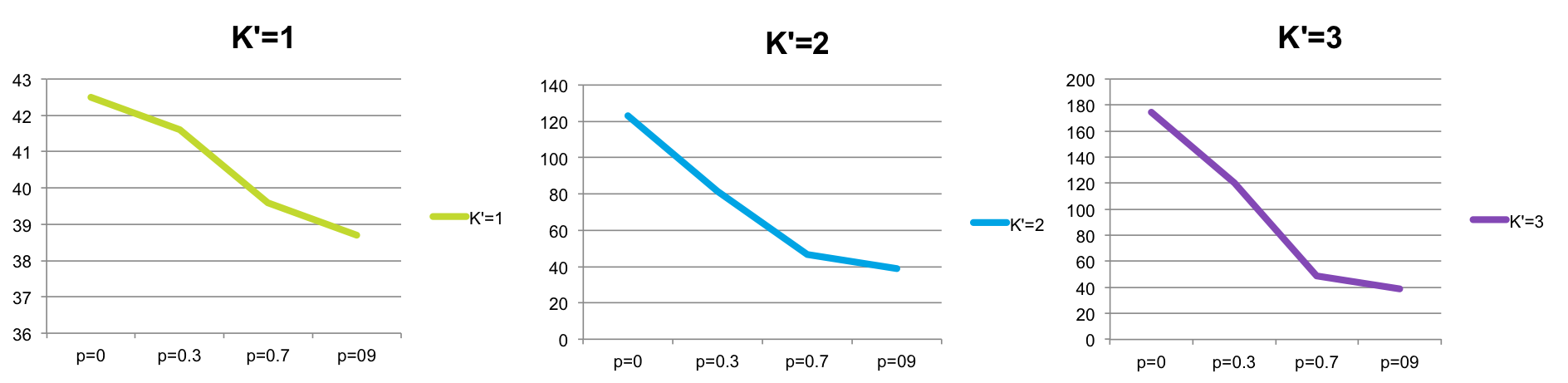}
\caption{HumanEva TGP different K' as overlap increase, ,  M=800}
\label{fig_K_each}
\end{figure*}

\begin{figure}[t!]
\includegraphics[width=0.5\textwidth]{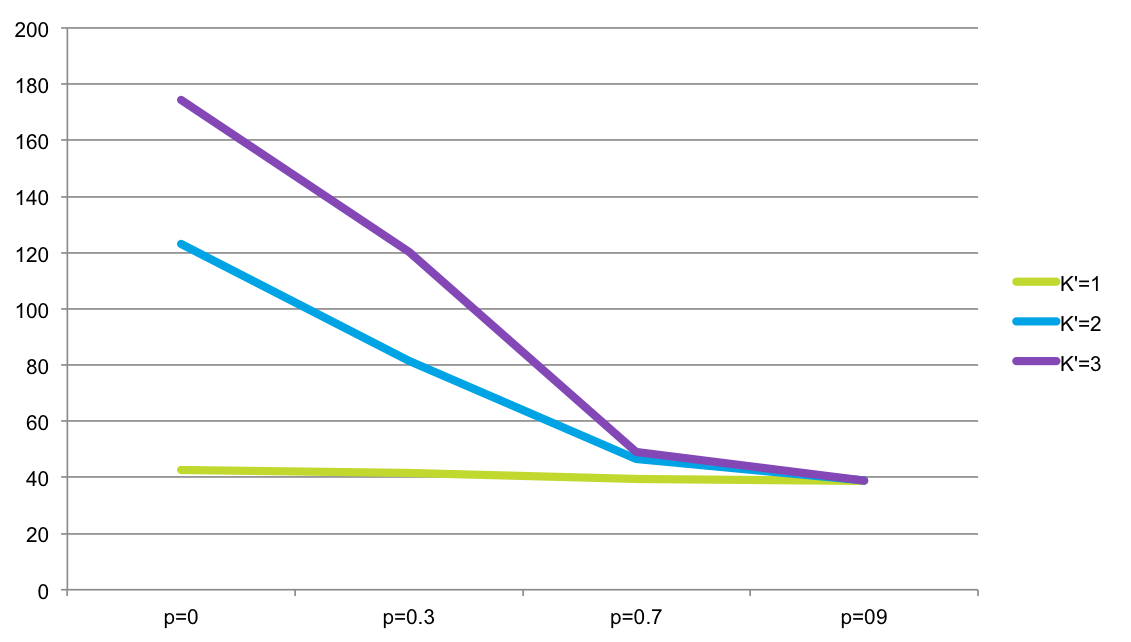}
\caption{Increasing K' significantly heart the performance for small overlap (Human Eva TGP,  M=800)}
\label{fig_K_all}
 % \label{fig:SpeedUp}
\end{figure}

\begin{figure*}[h!]
\centering
\begin{subfigure}[b]{1.0\textwidth}
 \includegraphics[width=1.0\textwidth,height=0.34\textwidth]{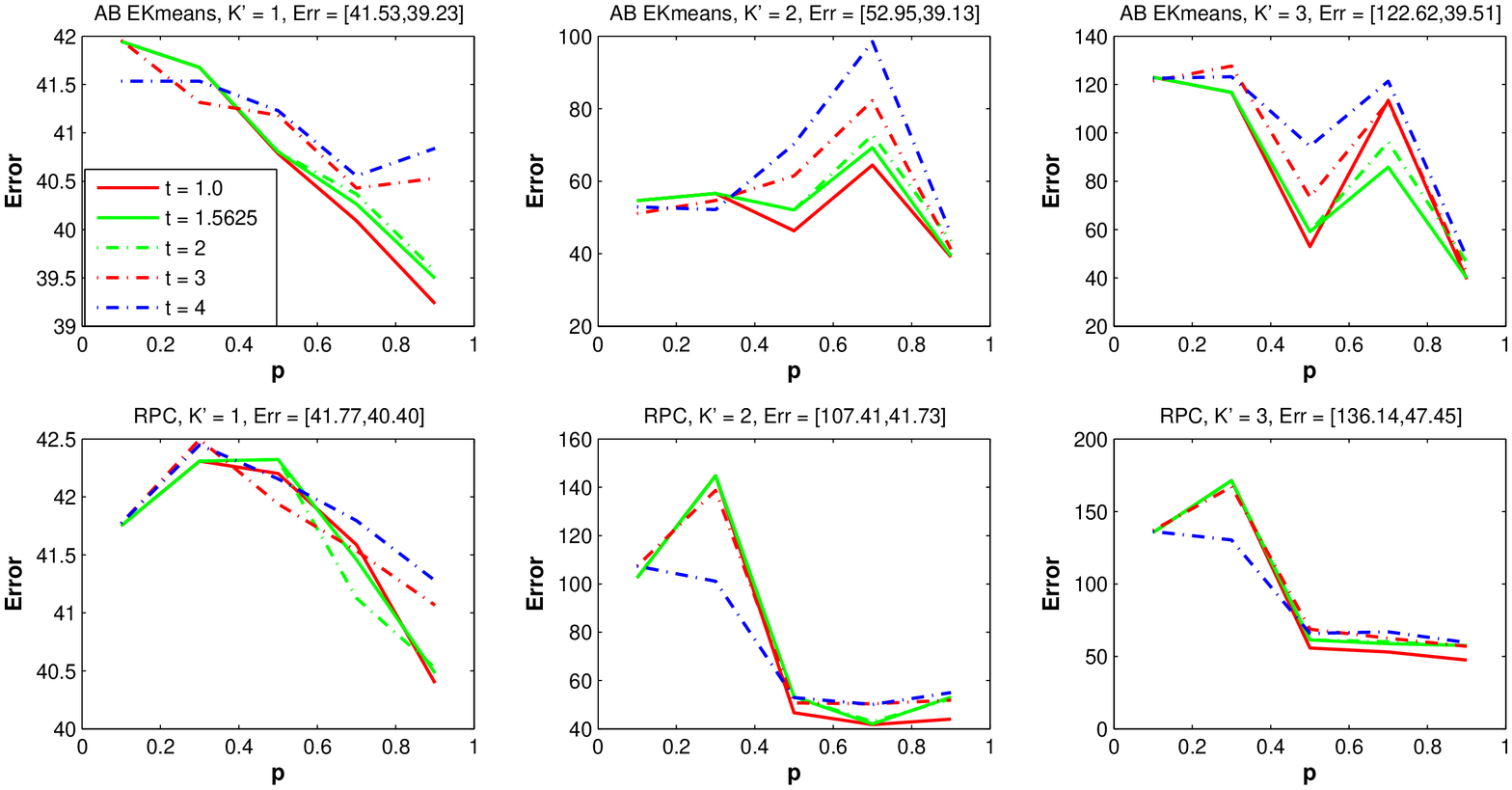}
  \caption{TGP-ODC (M=400)}
\end{subfigure}
\begin{subfigure}[b]{1.0\textwidth} \includegraphics[width=1.0\textwidth,height=0.34\textwidth]{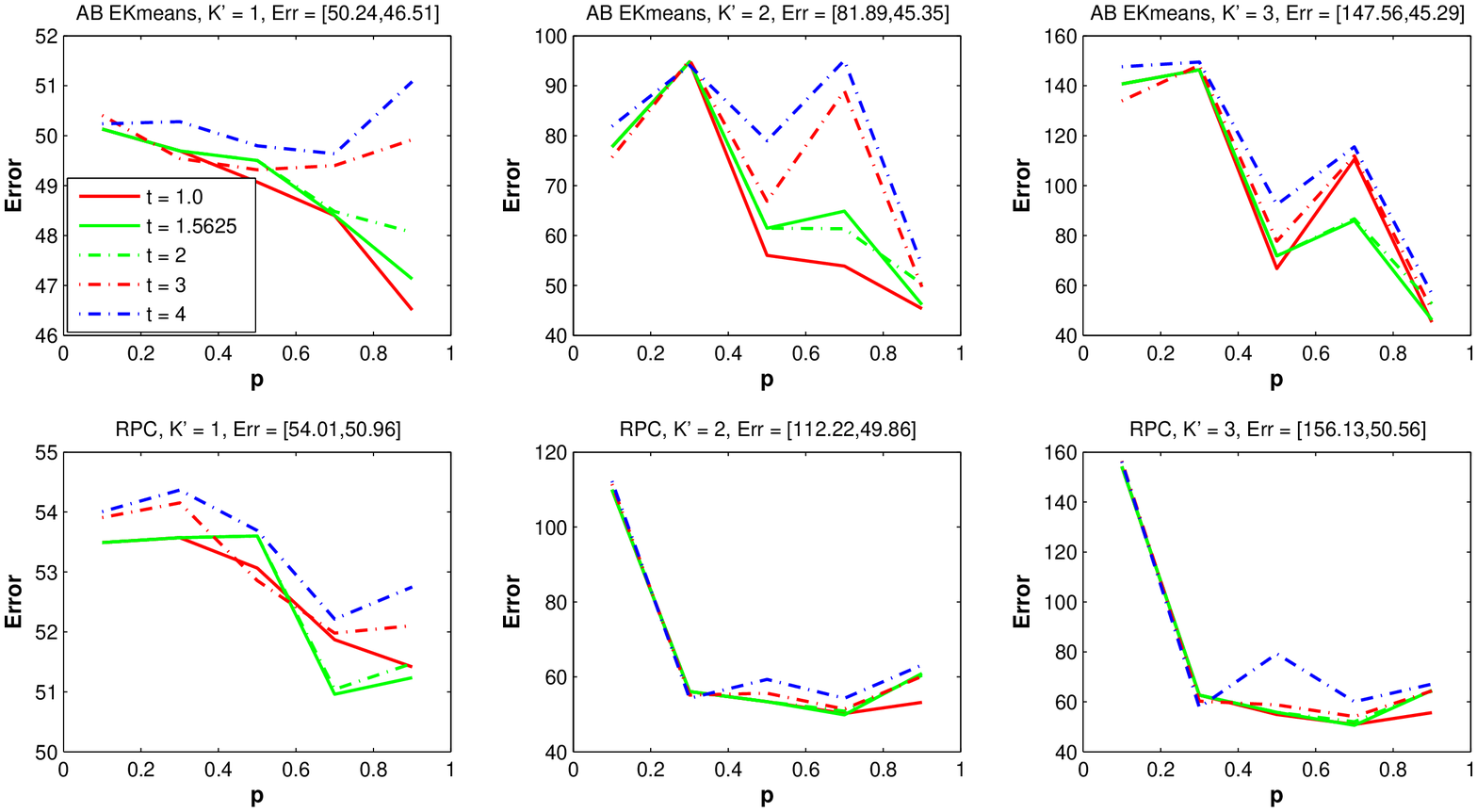}
\caption{GPR-ODC (M=400)}
\end{subfigure}
\caption{Overlapping Domain Cover Parameter Analysis of GPR and TGP  on Human Eva Dataset (best seen in color) (M=400)}
\label{fig:ODCAnalysis400}
%\label{fig:teaser}
\end{figure*}

This analysis helped us conclude recommending choosing  $t$ close to $1$, big overlap ($p$ closer to 1), and $K'=1$ is sufficient for accurate prediction. \ignore{
\begin{wrapfigure}{r}{0.3\textwidth}
\centering
  \vspace{-10pt}
  \includegraphics[width=0.35\textwidth]{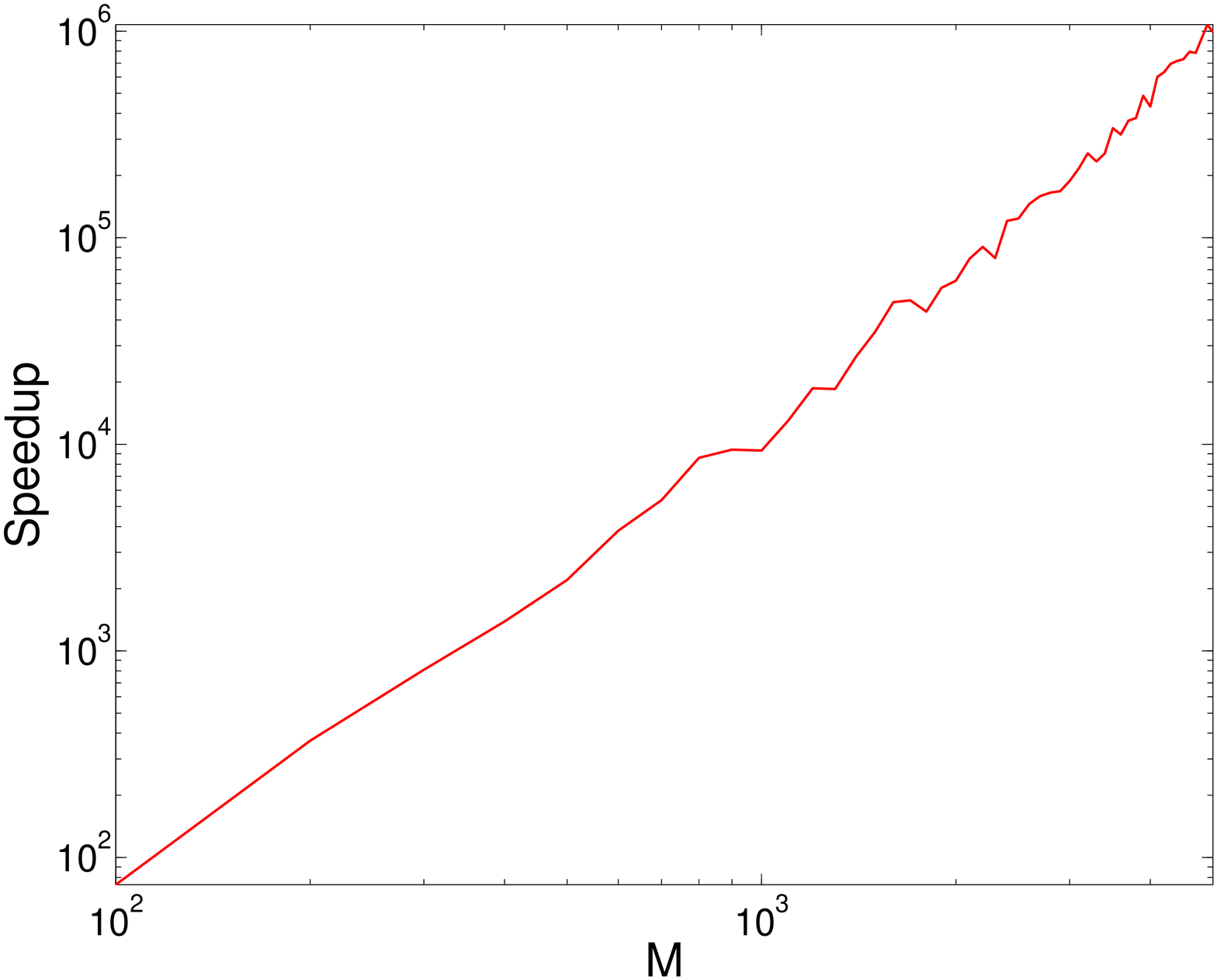}
  \vspace{-10pt}
  \caption{Matrix Inverse Precomputation Speedup of ODC framework prediction as $M$ increases (log-log scale)}
  \label{fig:SpeedUp}
\end{wrapfigure}}

%\begin{wrapfigure}{r}{6.5cm}\vspace{-6mm}
%\includegraphics[width=6.5cm]{figMInvSpeedUp.eps}
%\caption{Matrix Inverse Precomputation Speedup of ODC framework prediction as $M$ increases (log-log scale)}\vspace{-4mm}\label{fig:SpeedUp}\end{wrapfigure}
\begin{figure}[h!]
\centering
\includegraphics[width=0.5\textwidth,height=0.35\textwidth]{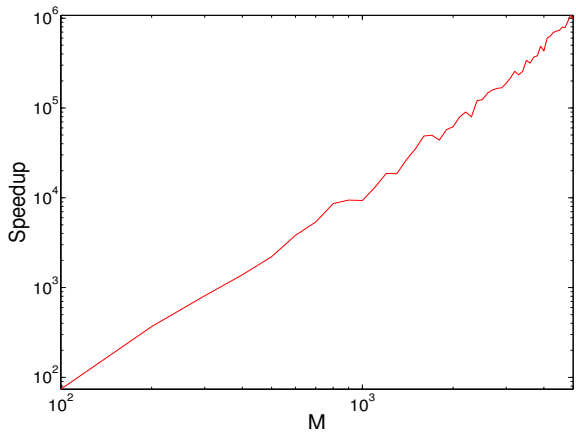}
  \caption{Speedup of ODC framework prediction on either TGP or GPR while retrieving precomputed matrix inverses as $M$ increases, compared with computing them on test time by KNN scheme (log-log scale)}
  \label{fig:SpeedUp}
\end{figure}
Having accomplished the performance analysis which comprehensively interprets our parameters, we used the recommended setting to compare the performance with other methods and show the benefits of this framework. Figure~\ref{fig:SpeedUp}\ignore{start by showing} shows the speedup gained by retrieving the matrix inverses on test time, compared with computing them at test time by NN scheme. The figure shows significant speedup from precomputing local kernel machines.  
%\vspace{-1.5mm}

Table~\ref{tab:tblRes} shows error, training time and prediction time of NN, FIC, and different variations of ODC  on Poser and Human-Eva datasets. Training time is formatted as ($t_c$ + $t_p$),  where $t_c$ is the clustering time and $t_p$ is the remaining training time excluding  clustering.  As indicated in the top part of table ~\ref{tab:tblRes}, TGP  under our ODC-framwork can significantly speedup the prediction compared with NN-scheme in ~\cite{Bo:2010}, while achieving competitive performance; better in case Poser Dataset. As illustrated in our analysis in Figure~\ref{fig:ODCAnalysis}, higher overlap ($p$) gives better performance. From time analysis perspective, higher $p$ costs more training time due that more subdomains are created and trained.  While, Figure~\ref{fig:ODCAnalysis} and Table~\ref{tab:tblRes} indicates that AB-Ekmeans gives better performance than RPC under both GPR and TGP, AB-Ekmeans takes more time for clustering. Yet,  it is feasible to compute in all the datasets, we used in our experiments. Our experiments also indicate that as $p \to 1$ in TGP and GPR,    $K'=2$ and $K'=3$ takes double and triple the prediction time respectively, compared with $K'=1$, with almost no error reduction\ignore{; see detailed table in the SM}. We also compared our model to FIC in case of GPR, and our model achieved smaller error and smaller prediction time; see bottom part in Table~\ref{tab:tblRes}. However, TGP consistently gives better results on both Poser and HumanEva datasets. We also tried full TGP and GPR on Poser and Human Eva Datasets. Full TGP error is $5.35$ for Poser and $40.3$ for Human Eva. Full GPR error is $6.10$ for Poser and $59.62$ for Human Eva. The results indicate that ODC achieves either better or competitive to the full models. Meanwhie, the speedup is sigbificant for TGP prediction ($21X$ for Human Eva and $11X$  for Poser Datasets); see Fig.~\ref{fig_TGP_speedup_hp}. For GPR prediction, we achieved the best performance and the lowest prediction time compared to existing GPR prediction methods; see Fig.~\ref{fig_GPR_speedup_hp}.

Based on our comprehensive experiments on HumanEva and Poser datasets, we conducted an experiment on Human3.6M  dataset with TGP kernel machine, where $M = 1390$, $t=1$,  $p = 0.6, K'=1$, Ekmeans for clustering. We achieved a speedup of ${41.7X}$ on prediction time using our ODC framework  compared with NN-scheme, i.e., $7$ days if NN-scheme is used versus $4.03$ hours in our case with our MATLAB implementation. The error is $13.5$ (cm) for NN and $13.8$ (cm) for ODC; see Fig.~\ref{fig_TGP_speedup_h36}.

\begin{figure}
\includegraphics[width=0.5\textwidth]{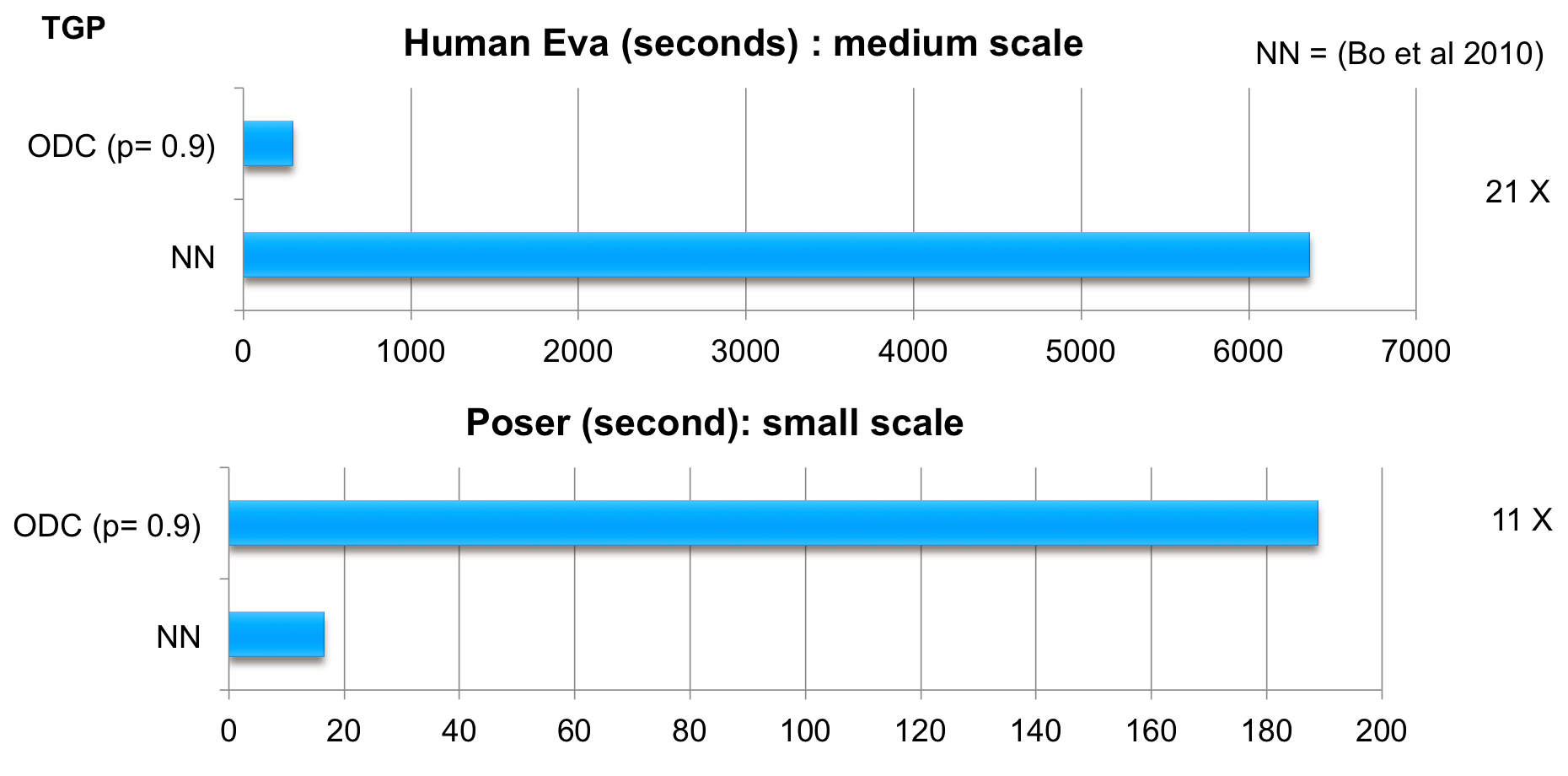}
\caption{TGP Human Eva Dataset (Speed)}
\label{fig_TGP_speedup_hp}
\end{figure}

\begin{figure}
\includegraphics[width=0.5\textwidth]{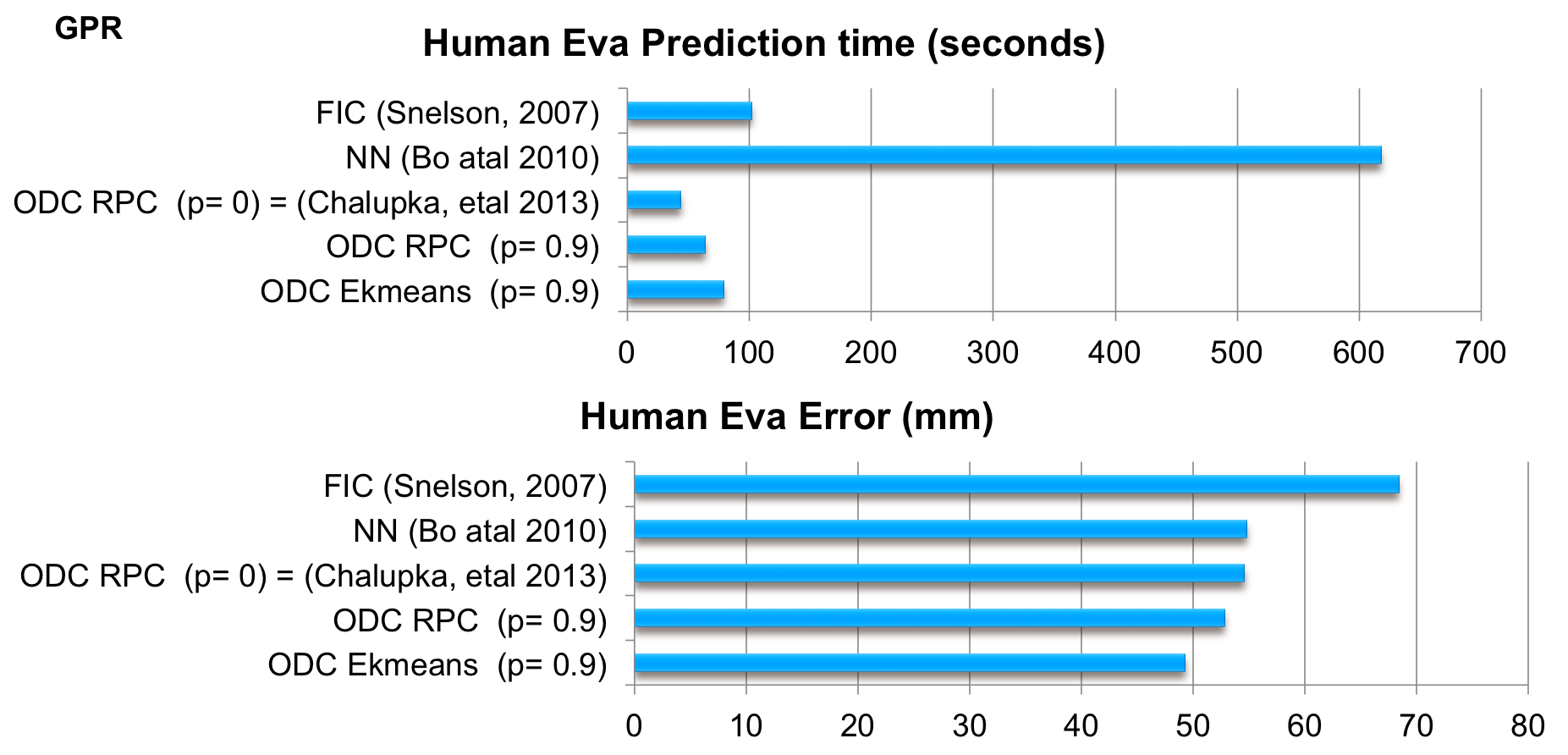}
\caption{GPR speed and error (Human Eva Dataset)}
\label{fig_GPR_speedup_hp}
\end{figure}

\begin{figure}
\includegraphics[width=0.5\textwidth]{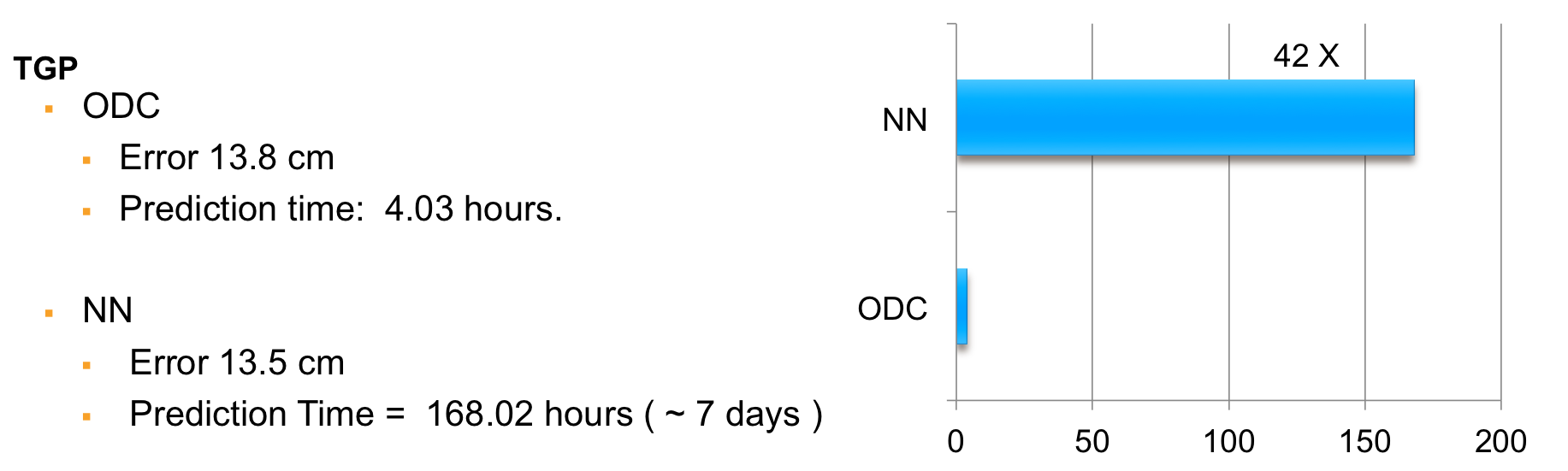}
\caption{TGP speed and error (Human3.6M Dataset)}
\label{fig_TGP_speedup_h36}
\end{figure}

\ignore{
We compare the performance of the proposed methods TGP-ODC and IWTGP-ODC against TGP-KNN \cite{Bo:2010} and IWTGP-KNN \cite{Yamada:2012}, respectively. We report performance on three publicly available datasets: Poser \cite{Trig06}, HumanEva \cite{SigalBB10} and Human3.6M \cite{human3m}. 
\begin{table}[htbp!]
  \centering
      \vspace{4pt}
  \caption{Error and Time for Poser and Human Eva datasets (on 2.6GHZ intel core i7), M = 800}
    \vspace{3pt}
    \scalebox{0.5}{
    \begin{tabular}{|l|l|lll|lll|}%ccc}
    \toprule
          &   & \textbf{Poser} & \textbf{} &       & \textbf{HumanEva} &       &      \\% & \textbf{Human 3.6} &       &  \\
    \midrule
         & & \textbf{Error (deg)} & \textbf{Training Time} & \textbf{Prediction Time} & \textbf{Error (mm) } & \textbf{Training Time} & \textbf{Prediction Time} \\%& \textbf{Error} & \textbf{Training Time} & \textbf{Prediction Time} \\
 \textbf{TGP}   & \textbf{NN} & 5.43       &  -     &    188.99 sec   & \textbf{38.1}      &  -     & 6363.823 sec \\%       &       &       &  \\
  & \textbf{ODC ($p= 0.9, t=1, K'=1$)-Ekmeans} & \textbf{5.4 }     &      (3.7 +25.1 ) sec  &   \textbf{16.5}  sec  &    \textbf{38.99 }  &  (2001 + 45.4) sec    &  \textbf{298.4} sec\\%       &       &       &  \\
    & \textbf{ODC ($p= 0.9, t=1, K'=2$)-Ekmeans} & {5.53}     &      (3.7 +29.46 ) sec  &   47.04  sec  &    {39.2 }  &  (2001 + 45.24) sec    &  569.6946  sec\\%       &       &       &  \\
      & \textbf{ODC ($p= 0.9, t=1, K'=3$)-Ekmeans} & {5.4}     &      (3.7 +28.8 ) sec  &   71.4 sec  &    {40.9 }  &  (2001 +  45.7) sec    &  721.0 sec\\%       &       &       &  \\
    & \textbf{ODC ($p= 0, t=1, K'=1$)-Ekmeans} &   7.6    &    (3.9 + 1.33) sec   &   14.8 sec  &    41.87   & (240 + 4.9832 ) sec       &  256.7709 \\%      &       &       &  \\
       & \textbf{ODC ($p= 0, t=1, K'=2$)-Ekmeans} &   12.3    &    (3.9 + 2.69) sec   &   42.25 sec  &    136.52 & (240 + 4.7790  ) sec       &  514.93 \\%      &       &       &  \\
       & \textbf{ODC ($p= 0, t=1, K'=3$)-Ekmeans} &   12.52    &    (3.9 + 1.86) sec   &    72.38 sec  &    187.72   & (240 + 4.7547 ) sec       &  770.9781 \\%      &       &       &  \\
      &  \textbf{ODC ($p= 0.9, t=1, K'=1$)-RPC} & 5.6      &      (0.23 +41.6 ) sec  &   15.8 sec  &   39.9  &    ( 0.45 + 49.05) sec     & 277.25 sec\\%       &       &       &  \\
         &  \textbf{ODC ($p= 0.9, t=1, K'=2$)-RPC} & 5.52      &      (0.23 +43.80 ) sec  &   43.802 sec  &   40.41  &    ( 0.45 + 46.77) sec     & 677.52 sec\\%       &       &       &  \\
        &  \textbf{ODC ($p= 0.9, t=1, K'=3$)-RPC} & 5.59  &      (0.23 +43.05 ) sec  &   67.11 sec  &    41.21&    ( 0.45 + 47.63) sec     &  882.6725 sec\\%       &       &       &  \\
  &  \textbf{ODC ($p= 0, t=1, K'=1$)-RPC} &   7.7    &   (0.15 + 1.717) sec   &   13.89 sec  &  42.32    &  (0.19 + 5.2551)      sec &  241.64 sec\\%      &       &       &  \\
    &  \textbf{ODC ($p= 0, t=1, K'=2$)-RPC} &   9.29 &   (0.15 + 1.83) sec   &   41.86 sec  &  58.99    &  (0.19 + 5.16)      sec &  475.14 sec\\%      &       &       &  \\
      &  \textbf{ODC ($p= 0, t=1, K'=3$)-RPC} &   12.47    &   (0.15 + 1.80) sec   &   66.42 sec  &  136.49    &  (0.19 + 5.20)      sec &  721.49 sec\\%      &       &       &  \\
      \hline
  \textbf{GPR}  & \textbf{NN} & 6.77      &   -    &  24 sec     &   54.8    &      - &    617.8  sec \\%&       &       &  \\
   & \textbf{ODC ($p= 0.9, t=1 , K'=1$)-Ekmeans} &  \textbf{6.27}      &  (3.7 +11.1 ) sec  &       \textbf{0.56}  sec & \textbf{49.3}  &   (2001 + 42.85)sec & \textbf{78.85} sec \\%       &       &       &  \\
   & \textbf{ODC($p= 0.0, t=1 , K'=1$)-Ekmeans} & 7.54      &   ( 3.9 + 1.38 sec) &    0.35 sec   & 49.6  &  (240 + 6.4) sec  &  48.1 sec\\%      &       &       &  \\
  & \textbf{ODC ($p= 0.9, t=1 , K'=1$)-RPC} &  6.45      &  (0.23 +17.3 ) sec  &       0.52  sec & 52.8  & (0.49 + 46.06) sec     &  64.13 sec\\%       &       &       &  \\
    & \textbf{ODC ($p= 0.0, t=1 , K'=1$)-RPC  = ~\cite{Chalupka:2013}} &   7.46    &   (0.15 + 1.47) sec &    0.27 sec   & 54.6  &  (0.261 + 4.58 ) sec & 43.52 sec\\%      &       &       &  \\
    & \textbf{FIC ~\cite{fic06}} &   7.63 (+/- 0.4)  &   (- + 20.63)   &    0.3106     &   68.36(+/- 0.84)    &  -     & 101.5442 (+/- 1.36) sec\\%      &       &       &  \\
    \bottomrule
    \end{tabular}}%
  \label{tab:tblRes}%
\end{table}%
}
%\textbf{ODC-Parameters}

%\input{experimentssubold}

\section{Conclusion}
\label{sec:7}
We proposed an efficient ODC framework for kernel machines and validated the framework on structured regression machines on three human pose estimation datasets. The key idea is to equally  partition the data and create cohesive overlapping subdomains, where local kernel machines are computed for each of them. The framework is general and could be applied to various kernel machine beyond GPR, TGP, IWTGP validated in this work. Similar to TGP and IWTGP, our framework could be easily applied to the recently proposed Generalized  TGP~\cite{SMTGP_Elhoseiny15} which is based on Sharma Mittal divergence, a relative entropy measure brought from Physics community. We also theoretically justified  our framework's notion. 

\noindent \textbf{Acknowledgment.} This research was partially funded by NSF award \# 1409683.

\ignore{In this paper, we presented an ODC framework for structured regression kernel machines including GPR and TGP, applied to  3D pose estimation. The proposed model preserves accuracy while achieving a significant speedup in prediction, which is crucial for large scale prediction problems. The approach was evaluated on Poser, Human Eva and Human3.6M datasets. The key idea is to equally  partition the data and create cohesive overlapping subdomains, where local kernel machines are computed for each of them. The final prediction is then computed the predictions of the closest kernel machines. We should some interesting properties of the proposed framework including its applicability to a variety of the state-of the-art kernel methods.}

%\begin{abstract}
%\input{abstract}
%\keywords{Text Visualization \and Multi-level MindMap Automation}
%\end{abstract}

%%%%%%%%% BODY TEXT

%\section{Kernel Conclusion}
%\input{Kconclusion}

%\bibliographystyle{ACM-Reference-Format-Journals}

%\begin{acknowledgements}
%\end{acknowledgements}

\bibliographystyle{spmpsci}      % mathematics and
%\bibliography{acmsmall-sample-bibfile}
%\bibliography{write_a_classifier}
\bibliography{egbib.bib}

\begin{appendices}

\section{IWTGP-ODC Experiments}
Tables ~\ref{tab:poserw} and ~\ref{tab:hevaresw} details the results of IWTGP-ODC experiments on Poser and HumanEva datasets in terms of error and speedup in prediction time.

\begin{table}[htbp]
  \centering
    \scalebox{0.8}
  {
     \begin{tabular}{|c|cc|}
      \hline
    \multirow{2}[4]{*}{\textbf{}} & \textbf{IWTTGP} & \multicolumn{1}{c|}{\textbf{IWTGP-ODC}} \\
     	& ($M = 800, M_{tst} = 418$)  & ($M = 800, M_{tst} = 418$)	\\ \hline
    \textbf{error (deg)} & 6.1 &  5.32 \\
    \textbf{err reduction (deg)} & -     &  \textbf{0.783 }\\
    \textbf{err reduction \%} & -     & \textbf{12.836\%} \\
    \textbf{Prediction Time (sec)} & 360.0  & 26.61 \\
    \textbf{speedUp} & -      & \textbf{13.5} \\ \hline
    \end{tabular}%     
    }
  \caption{POSER dataset IWTGP-NN vs  IWTGP-ODC}
 \label{tab:poserw}%
\end{table}%

\begin{table}[htbp]
  \centering
  \scalebox{0.9}
  {
\begin{tabular}{|c|cc|}
\hline
    \multirow{2}[4]{*}{\textbf{}} & \textbf{IWTGP} & \multicolumn{1}{c|}{\textbf{IWTTGP-ODC}} \\
    & ($M = M_{tst} = 800$)& ($M=M_{tst}=800$)\\  \hline
          & \textbf{}  & \textbf{ } \\
    \textbf{error (mm)} & 39.1  & 39.3 \\
    \textbf{err reduction (mm)} & -     & \textbf{\textbf{-0.2}} \\
    \textbf{err reduction \%} & -   & \textbf{\textbf{-0.512\%}} \\
    \textbf{Prediction Time (sec)} & 7938.15 &  569.66 \\
    \textbf{speedUp} & -     &  \textbf{13.92}\\ \hline
    \end{tabular}%
    }
 \caption{Humen Eva dataset: IWTGPKNN vs IWTGP-ODC}
  \label{tab:hevaresw}
\end{table}% 

%\section{More Figures and Results}
%figure ~\ref{fig:ODCAnalysis400} shows our analysis on ODC for M =400. 

\ignore{
\begin{figure}[h!]
\centering
\begin{tabular}{c}
\bmvaHangBox{{  \includegraphics[width=1.0\textwidth,height=0.34\textwidth]{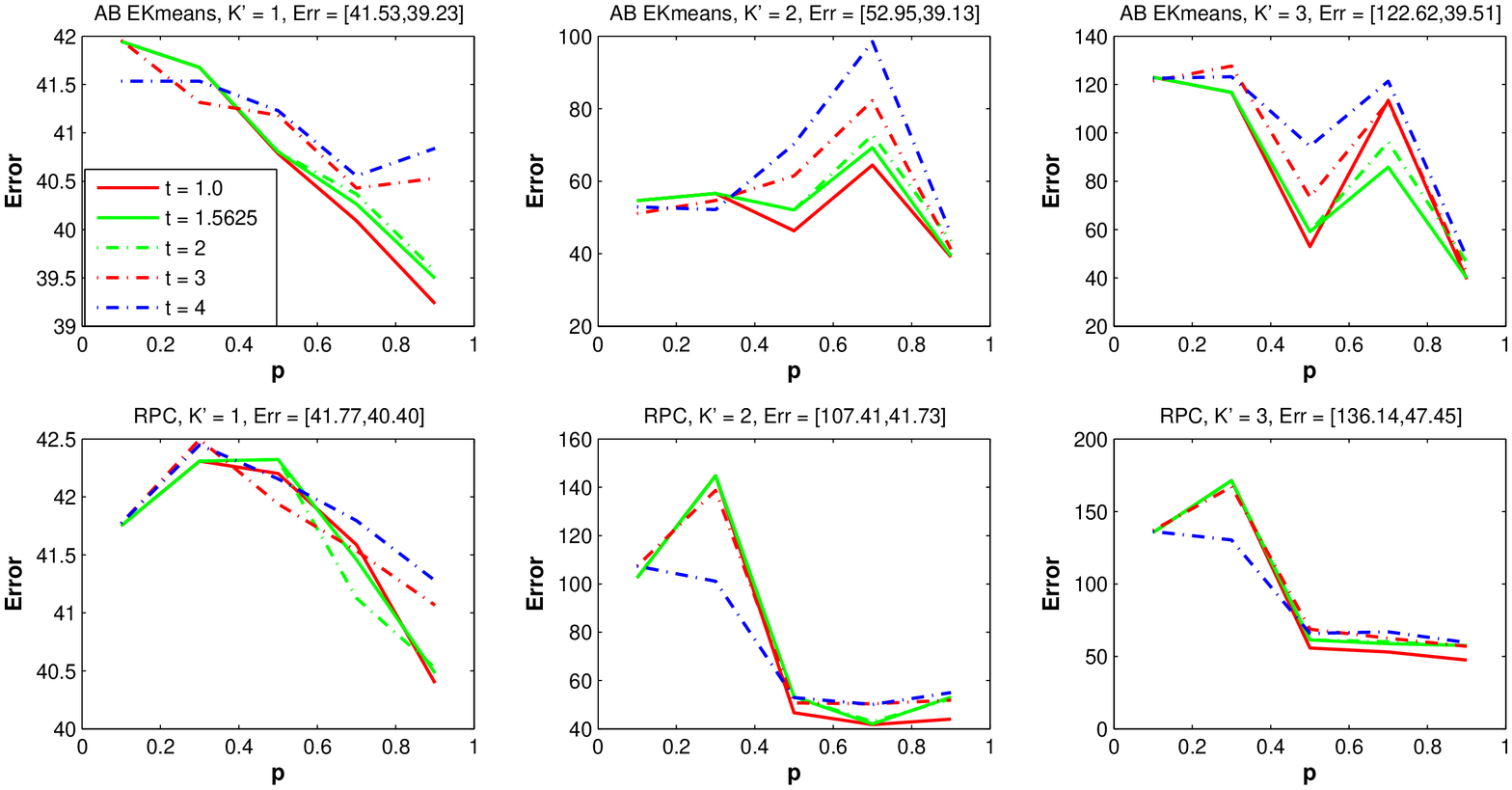}}}\\
(a) TGP-ODC (M=400) \\
\bmvaHangBox{{  \includegraphics[width=1.0\textwidth,height=0.34\textwidth]{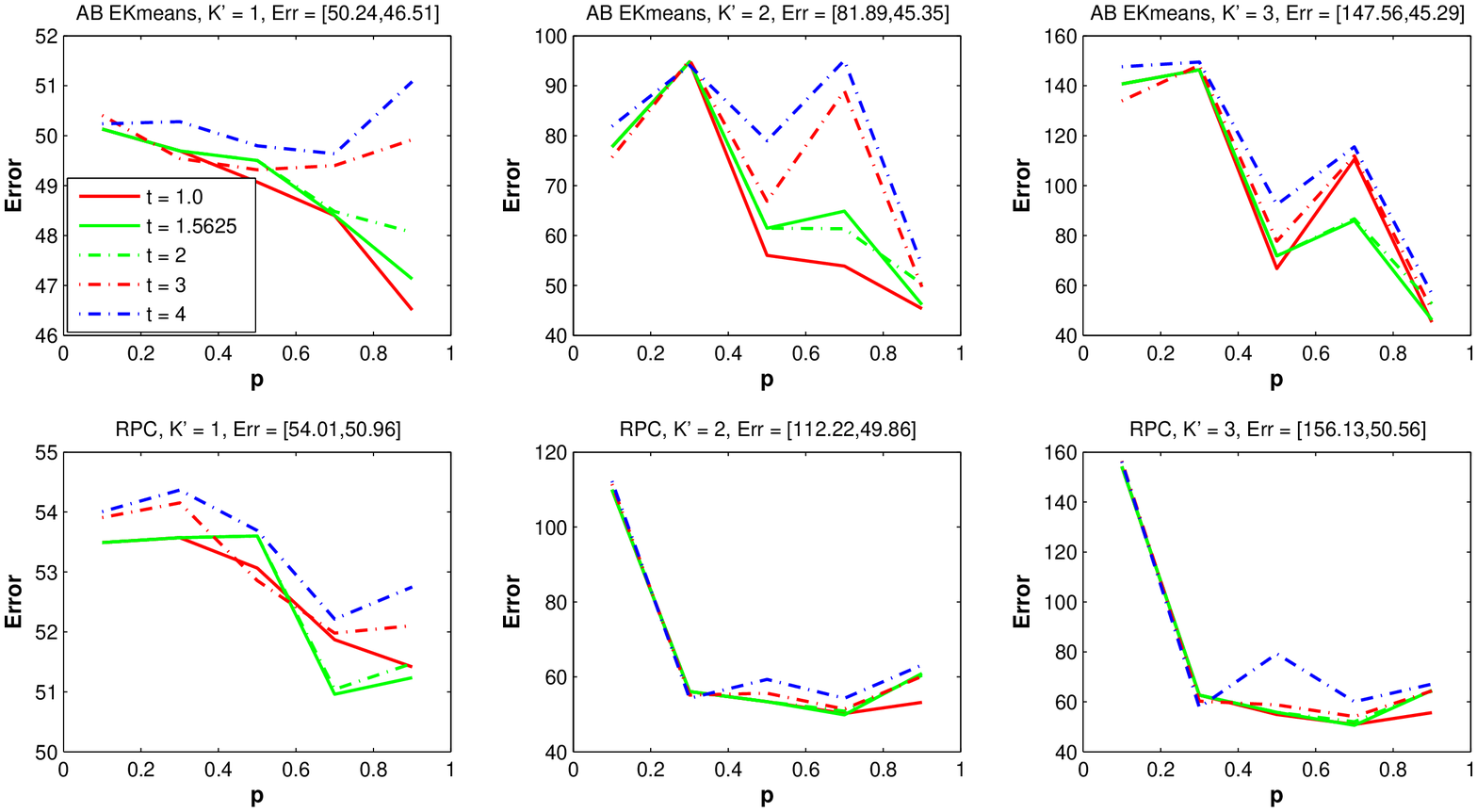}}}\\
(a) GPR-ODC (M=400) \\
\end{tabular}
\vspace{2mm}
\caption{Overlapping Domain Cover Parameter Analysis of GPR and TGP  on Human Eva Dataset (best seen in color) (M=400)}
\label{fig:ODCAnalysis400}
%\label{fig:teaser}
\end{figure}}

\section{More figures on AB Ekmeans}

 Figure~\ref{fig:ekmeans} shows the clustering performance on 300000 random 2D point (K=5). Figure  ~\ref{fig:hevaVis3} shows the clustering output of our algorithm visualized on using the first three principal components of Human Eva training hog features. The figures shows that the cluster are spatially cohesive but not necessarily circular. This makes the elliptic distribution of the data captured by Mode 3 gives more accuracy membership measure me to the subdomains. %\ignore{That's why Mode 3 gives a better accuracy than Mod 1. It also gives an accuracy better than Mode 2. A problem in Mode 2 is how to choose K so that it relects the correct associativity of the test point to the closest subdomans.}
 
 \ignore{\begin{figure}[h!]
\begin{tabular}{ccc}
\bmvaHangBox{\fbox{\includegraphics[width=4.6cm]{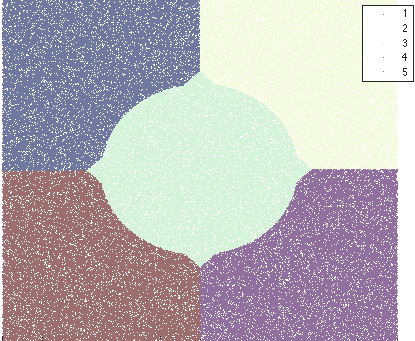}}}&
\bmvaHangBox{\fbox{\includegraphics[width=5.2cm]{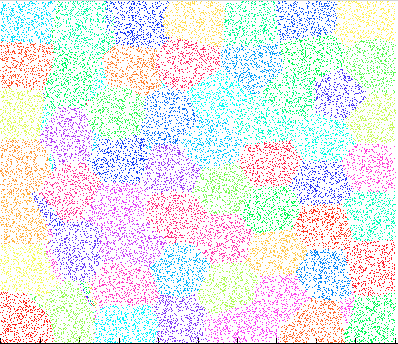}}}\\
(a)&(b)
\end{tabular}
\caption{Applying our Assign and Balance variant of Kmeans on 300,000 random 2D points:  5 clusters}
\label{fig:ekmeans}
%\label{fig:teaser}
\end{figure}}

 \begin{figure}[h!]
 \centering
\begin{subfigure}[b]{0.5\textwidth}
\includegraphics[width=6.6cm]{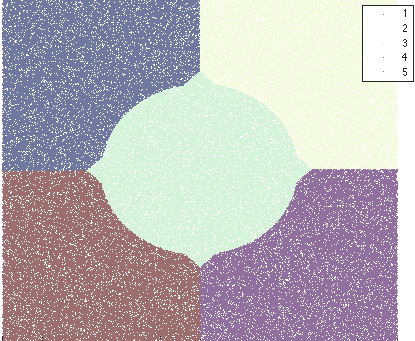}
  \caption{5 clusters}
\end{subfigure}
%\begin{subfigure}[b]{0.5\textwidth}
%\includegraphics[width=5.2cm]{Ekmeans57.png}
  %\caption{57 clusters}
%\end{subfigure}
\caption{Applying our Assign and Balance variant of Kmeans on 300,000 random 2D points}
\label{fig:ekmeans}
%\label{fig:teaser}
\end{figure}
\begin{figure}[h!]
\centering
 \includegraphics[width=0.5\textwidth]{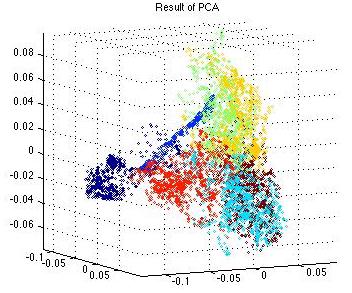}
\caption{Human Eva clustering first three Pricipal Components }
\label{fig:hevaVis3}
\end{figure}

%\subsection{Equal Size Assignment Algorithms Pseudocode}
%As presented in the pape, our k-means variant algorithms modifies only the assignment step of  the standard k-means algorithm. Algorithm ~\ref{alg:ddclusterALg2} and ~\ref{alg:ddclusterALg1} shows the pseudo-code of the assignments steps on IMDA-k-means and AB-k-means algorithms respectively. We attach the MATLAB implementation of both algorithms in "Ekmeans-assign" folder, We plan to release the whole implementation of our paper as well as soon as we well-document of the code. 

\section{Overlapping Domain Cover(ODC) Generation-Algorithm}

Algorithm ~\ref{alg:sdgen} shows how the overlapping sub-domains are generated form the the equal size clusters from the closest $r$ clusters. \ignore{ if the retrieved nearest neighbor points  belongs to more than $^OC_C$ clusters.}
\begin{algorithm}
{\textbf{Input:} Clusters ${\{C_k\}}_{k=1}^{K} $}
\KwOut{Overlapping subdomains ${\{D_k\}}_{k=1}^{K}$}
\ForEach{Cluster $C_k$}{
Compute the closest $r$ clusters ${\{{{C^{'}}_i}\}}_{i=1}^{r}$ based on $DK_i = \| \mu_k- \mu_i \|$ , $i\neq k$\\
Let $LK_i = 1/DK_i,  {WK_i} =  \frac{LK_i}{\sum_{l=1}^{^OC_C} LK_l}$  ${i=1 : r}$\\
Let ${NPK_i} =  floor(WK_i * OPC)$, ${i=1 : r}$ \\
Let $ExKPts = (1-p) M - \sum_{l=1}^{r} NPK_l$ \\
Let ${NPK_i}$ =  $NPK_i +1$ , $ i=1 : ExKPts $\\
$D_k =  C_k$ \\
Let $overflow = 0$\\
  \Comment{The following for loop goes over the $r$ clusters on an increasing order of $DK_i$ }\\
\For{i=1 : $r$} {
  \If{${NPK_i}$> $|C_i|$} { $overflow = overflow+ {NPK_i} -|C_i|$ \\  $NPK_i = |C_i|$   } 
  \If{${NPK_i}$< $|C_i|$} { $G_i =min(overflow, |C_i| -NPK_i$ ) \\  $NPK_i = NPK_i +G_i $ \\  $overflow =overflow-G_i$   } 
  %\ELSE{} {}\\
 $Ps_i = KNN({OVC_K}_j,NPK_i )$ 
 \\ $D_k = D_k \cup Ps_i$ 
 }

\For{i=1 : $r$} { $Ps_i = KNN({OVC_K}_j,NPK_i )$ \\ $D_k = D_k \cup Ps_i$ }

\Comment{where KNN is the K-nearest neighbors algorithms. For high performance calculation of $KNN$, we use FLANN \cite{flann09} to calculate $KNN$.}
}
\caption{Subdomains Generation (Note: All ${\{D_k\}}_{k=1}^{K}$ are stored as indices to $X$).  }
\label{alg:sdgen}
\end{algorithm}

%
%
%
%\subsection{Prediction} 
%From the above discussion, the prediction for each subdomain is computed as follows
%
%\begin{equation}
%\begin{split}
%\hat{Y^i_{x_*}} =  \underset{Y^i_{x_*}}{\operatorname{argmin       }}[ & k_Y(\textbf{Y}^i_{x_*},\textbf{Y}^i_{x_*}) -2 
%k_y(\textbf{Y}^i_{x_*})^T \textbf{u}_w -\\ & \eta_w  log (K_Y(\textbf{Y}^i_{x_*},\textbf{Y}^i_{x_*}) -\\& 
%k_y(\textbf{Y}^i_{x_*})^T {\textbf{W}^i}^\frac{1}{2} ({\textbf{W}^i}^\frac{1}{2} \textbf{K}_Y {\textbf{W}^i}^\frac{1}{2} +  \lambda_y I)^{-1} \\&{\textbf{W}^i}^\frac{1}{2} k_y(\textbf{Y}^i_{x_*}) ) ]
%\end{split}
%\end{equation}
%
%where $\textbf{u}_w = \textbf{W}^\frac{1}{2}  (\textbf{W}^\frac{1}{2} \textbf{K}_X \textbf{W}^\frac{1}{2} + \lambda_x I)^{-1} \textbf{W}^\frac{1}{2} k_x(\textbf{x})$, $\eta_w = k_X(\textbf{x},\textbf{x}) - k_x(\textbf{x})^T \textbf{u}_w$, $({\textbf{W}^i}^\frac{1}{2} \textbf{K}_X^i {\textbf{W}^i}^\frac{1}{2} + \lambda_x \textbf{I})^{-1}$,  $({{\textbf{W}^i}}^\frac{1}{2} \textbf{K}_X {\textbf{W}^i}^\frac{1}{2} + \lambda_x \textbf{I})^{-1}$ could be computed in quadratic time given  $\mathcal{M}^i$ and $W$. Hence,  the $\hat{Y^i_{x_*}}_j$ has  $O(iters \cdot M)^2$ complexity, where $iters$ is the number of iterations. 

\section{Local Kernel Machines hyper-parameters on each dataset}
The hyper parameters were learnt using cross validation on the training set for  GPR, TGP and IWTGP that we are interested in. The following subsection present the learnt hyper-parameters and the error measures on each dataset in case of TGPs.
\subsection{Poser Dataset}
The parameters $2 \rho_x^2$, $2 \rho_y^2$, $\lambda_X$, and $\lambda_Y$ were assigned to  $5$, $5000$, $10^{-4}$, and $10^{-4}$, respectively. 

\subsection{HumanEva Dataset}

The parameters $2 \rho_x^2$, $2 \rho_y^2$, $\lambda_X$, and $\lambda_Y$ were assigned to  $5$, $500000$, $10^{-3}$, and $10^{-3}$, respectively.

\subsection{Human 3.6 Dataset}

The parameters $2 \rho_x^2$, $2 \rho_y^2$, $\lambda_X$, and $\lambda_Y$ were assigned to  $5$, $500000$, $10^{-3}$, and $10^{-3}$, respectively.

%\section{Equal Size Kmeans (EKmeans): More Details}

\end{appendices}

\clearpage
%begin{comment}
\begin{wrapfigure}{l}{0.2\textwidth}
     \includegraphics[width=0.2\textwidth]{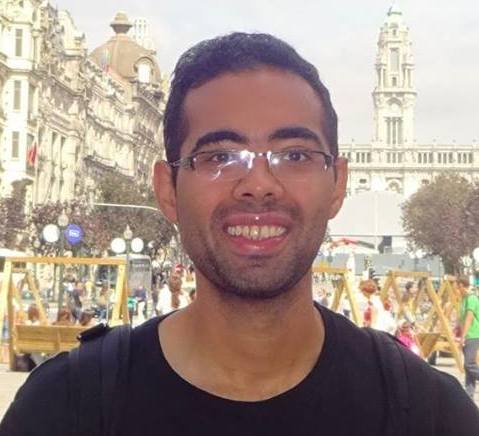}
\end{wrapfigure}
\textbf{Mohamed Elhoseiny } is a PostDoc Researcher at Facebook Research.
His primary research interest is in computer vision, machine learning, intersection between natural language and vision, language guided visual-perception,  and visual reasoning, art \& AI. He received his PhD degree from Rutgers University, New Brunswick, in 2016 under Prof. Ahmed Elgammal. Mohamed received an NSF Fellowship in 2014 for the Write-a-Classifier project (ICCV13), best intern award at SRI International 2014, and the Doctoral Consortium award at CVPR 2016.
\vspace{2mm}
 \begin{wrapfigure}{l}{0.2\textwidth}  
\vspace{-5mm} 
         \includegraphics[width=0.2\textwidth]{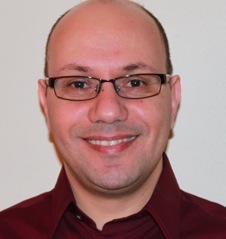}
         \vspace{-10mm} 
\end{wrapfigure}\,\,\;

 \textbf{Ahmed Elgammal} is a professor at the Department of Computer Science, Rutgers, the State University of New Jersey Since Fall 2002. Dr. Elgammal is also a member of the Center for Computational Biomedicine Imaging and Modeling (CBIM). His primary research interest is computer vision and machine learning. His research focus includes human activity recognition, human motion analysis, tracking, human identification, and statistical methods for computer vision. Dr. Elgammal received the National Science Foundation CAREER Award in 2006. Dr. Elgammal has been the Principal Investigator and Co-Principal Investigator of several research projects in the areas of Human Motion Analysis, Gait Analysis, Tracking, Facial Expression Analysis and Scene Modeling; funded by NSF and ONR. Dr. Elgammal is Member of the review committee/board in several of the top conferences and journals in the computer vision field. Dr. Elgammal received his Ph.D. in 2002 from the University of Maryland, College Park. He is a senior IEEE member.

\end{document}
% end of file template.tex